\documentclass[opre,nonblindrev]{informs3}

\DoubleSpacedXI

\usepackage{endnotes}
\let\footnote=\endnote
\let\enotesize=\normalsize
\def\notesname{Endnotes}%
\def\makeenmark{\hbox to1.275em{\theenmark.\enskip\hss}}
\def\enoteformat{\rightskip0pt\leftskip0pt\parindent=1.275em
  \leavevmode\llap{\makeenmark}}

\usepackage[greek,english]{babel}
\usepackage[utf8]{inputenc} 
\usepackage[T1]{fontenc}    
\usepackage{hyperref}       
\usepackage{url}            
\usepackage{booktabs}       
\usepackage{nicefrac}       
\usepackage{microtype}      

\usepackage[export]{adjustbox}
\usepackage{float}
\usepackage{bbm,bm}
\usepackage{mdframed}
\usepackage{relsize}
\usepackage{tikz}
\usepackage{pgfplots}
\usepackage{caption,subcaption}
\usepackage{multirow}
\usepackage{xcolor}
\usepackage{dsfont}
\usepackage{amsmath,amssymb,amsfonts}

\newtheorem{ex}{Example}

\def\*#1{\bm{#1}}

\usepackage[ruled,vlined,linesnumbered,resetcount]{algorithm2e}
\RestyleAlgo{boxruled}
\LinesNumbered
\usepackage{algorithmic}

\makeatletter
\newcommand{\eqnum}{\refstepcounter{equation}\textup{\tagform@{\theequation}}}
\makeatother

\usepackage{natbib}
 \bibpunct[, ]{(}{)}{,}{a}{}{,}%
 \def\bibfont{\small}%
\TheoremsNumberedThrough     
\ECRepeatTheorems

\EquationsNumberedThrough    

\begin{document}


\RUNAUTHOR{Bertsimas and Digalakis Jr}

\RUNTITLE{Improving Stability in Decision Tree Models}

\TITLE{Improving Stability in Decision Tree Models}

\ARTICLEAUTHORS{%
\AUTHOR{Dimitris Bertsimas, Vassilis Digalakis Jr}
\AFF{Sloan School of Management and Operations Research Center, Massachusetts Institute of Technology, Cambridge, MA 02139\\ \EMAIL{dbertsim@mit.edu,vvdig@mit.edu}} 
} 

\ABSTRACT{
Owing to their inherently interpretable structure, decision trees are commonly used in applications where interpretability is essential. 
Recent work has focused on improving various aspects of decision trees, including their predictive power and robustness; however, their instability, albeit well-documented, has been addressed to a lesser extent.
In this paper, we take a step towards the stabilization of decision tree models through the lens of real-world health care applications due to the relevance of stability and interpretability in this space.
We introduce a new distance metric for decision trees and use it to determine a tree's level of stability. 
We propose a novel methodology to train stable decision trees and investigate the existence of trade-offs that are inherent to decision tree models --- including between stability, predictive power, and interpretability. We demonstrate the value of the proposed methodology through an extensive quantitative and qualitative analysis of six case studies from real-world health care applications, and we show that, on average, with a small 4.6\% decrease in predictive power, we gain a significant 38\% improvement in the model's stability.
}%



\maketitle

%


\section{Introduction} 

Machine learning (ML) algorithms get increasingly integrated into our lives \citep{jordan2015machine}. Traditionally, the primary objective of ML has been to improve predictive power. The recent adoption of ML in high-stakes applications, ranging from health care \citep{obermeyer2016predicting,char2018implementing} to climate change \citep{rolnick2022tackling,cheong2022artificial} has highlighted the need to achieve additional, often competing, objectives, including interpretability, robustness, and stability. 
Loosely defined, an interpretable ML model allows humans to have an understanding of the logic behind the model choices \citep{murdoch2019definitions}. 
A robust ML model is protected against noise and different types of uncertainties in the data \citep{xu2009robustness}. 
A stable ML model is not significantly affected by small changes in the data \citep{breiman1996heuristics}.

Decision tree models \citep{breiman1984classification} are commonly used in high-stakes applications. Owing to their graphical visualization and discrete and sequential structure, which mimics the human thought process by successively asking questions that adapt based on previous answers, decision trees are inherently interpretable. In the words of Leo Breiman \citep{breiman2001statistical}, ``on interpretability, trees rate an A$^+$.'' The excellent interpretability of decision tree models, however, does not come for free: they are known to suffer in terms of all other objectives. ``While trees rate an A$^+$ on interpretability, they are good, but not great, predictors. Give them, say, a B on prediction.'' As we discuss in Section\ref{ss:tree-limitations}, a lot of effort has been dedicated to materially improving decision trees' predictive power and robustness, while maintaining their interpretability \citep{bertsimas2017optimal}. Improving their stability has been addressed to a much lesser extent.

\begin{ex} \label{ex:healthcare}
Consider a health care setting where we build an interpretable ML model to support a physician's decision making. A commonly encountered situation is that the initial dataset is small but larger amounts of data become available over time as more patients' information gets recorded. Then, it is reasonable to consider retraining the ML model to boost its predictive performance and potentially incorporate new patterns in the data. However, upon retraining, the new model often could change notably, owing to the instability of the underlying ML algorithm. This situation would harm interpretability and hinder the adoption of the model by the physician.
\end{ex}

The focus of this paper is the study of the stability of decision tree models through the lens of real-world health care applications. We introduce a quantitative (and, what we believe, practical) way to measure a decision tree's stability. Building on this measure of stability, we propose a methodology that enables training more stable decision trees. We investigate the potential existence of trade-offs between the different objectives in ML which we previously described. Ultimately, we hope that this work is a step towards addressing long-standing issues concerning the stability of decision tree models, and will further enhance their adoption in high-stakes applications.

\subsection{Decision Tree Models and their Limitations} \label{ss:tree-limitations}

\paragraph{Improving predictive power and robustness.} As we already discussed, the good-but-not-great predictive power of decision trees trained using heuristic methods such as CART \citep{breiman1984classification} has been documented since early on after their inception. Recently, the use of mixed-integer optimization-based methods to learn globally optimal or near-optimal decision trees has led to notable improvements in predictive power --- to the extent that such trees often compete with state-of-the-art black-box models \citep{bertsimas2017optimal,bertsimas2019machine,aghaei2020learning,aglin2020learning,carrizosa2021mathematical}. Optimal decision trees exhibit improved stability compared to heuristic ones, especially when used in combination with feature selection techniques \citep{bertsimas2022backbone}, but still exhibit instability. Moreover, recent approaches have developed decision trees with additional desirable properties, including robustness to noise or adversarial perturbations in the data features \citep{justin2022optimal,moshkovitz2021connecting,bertsimas2019robust} and fairness \cite{aghaei2019learning}.

\paragraph{The source of instability.} In a typical model selection procedure, a ``best'' ML model is chosen from a collection of predictors obtained, e.g., using different hyperparameters. A procedure is called unstable if a small change in the data used to obtain the sequence of models can cause large changes in the best model \citep{breiman1996heuristics}. This instability occurs commonly when there are many different models crowded together that have similar predictive power, in which case a slight change in the data can cause a change from one model to another. The two models are close to each other in terms of predictive power, but can be distant in terms of their structure. \cite{breiman2001statistical} describes this situation as the ``Rashomon effect'' and the set of all such models as the ``Rashomon set;'' the term is derived from Akira Kurosawa's 1950 film ``Rashomon,'' in which a murder is described in four contradictory ways by four witnesses. \cite{xin2022exploring} propose a technique for enumerating the Rashomon set for decision trees --- without, however, any stability considerations.


\paragraph{Are stability and interpretability at odds?} \cite{xu2011sparse} show that, for sparse regression, stability and interpretability (owing to the sparsity of the model) are at odds with each other: sparser models are shown to be less stable. However, from a practical viewpoint, we observed in \cite{bertsimas2021slowly} that the stability of sparse regression models improves by enabling them to vary smoothly over, e.g., time and controlling the distance between the respective coefficients. 
In the context of decision trees, stabilization has been traditionally achieved at the expense of interpretability via the use of boosting \citep{freund1997decision,chen2016xgboost} or bagging \citep{breiman2001random, breiman1996bagging} --- improving stability was, in fact, the primary motivation in developing Random Forest. Once again quoting \cite{breiman1996heuristics}, ``While stable procedures have desirable properties, stabilization by averaging is not a panacea. An area that needs exploration is the possibility of stabilization of procedures by changing their structure instead of averaging. An interesting research issue we are exploring is whether there is a more stable single-tree version of CART.'' In this paper, we hope to make progress toward answering this question.

\subsection{Towards More Stable Decision Trees}

\paragraph{Improving stability.} Attempts to improve decision trees' stability have largely focused on (heuristically) tweaking either the learning algorithm or the model selection procedure. In the former case, \cite{last2002improving} and \cite{mirzamomen2017framework} propose (slightly) more stable decision tree variants, such as directed graphs or trees with hyperplane splits, respectively, and \cite{aluja2003stability} propose a series of tests to prevent internal instability in the tree-growing process.  In the latter case, \cite{shannon1999combining} define a probability distribution for an equivalence class of trees and select the maximum likelihood tree structure; \cite{bertsimas2022stable} develop stable classification models by optimizing the choice of training/validation split, but their work does not materially improve the stability of decision trees. We propose a decision tree learning algorithm-agnostic \textbf{stabilization methodology}, which, by explicitly quantifying stability, allows us to identify the most stable trees in a collection thereof.


\paragraph{Measuring stability.} The first step towards quantifying decision trees' stability comprises of measuring the distance between two trees. To do so, a first family of approaches \citep{turney1995bias,briand2009similarity}, referred to as ``semantic stability,'' evaluates the degree to which two decision trees make the same predictions, e.g., by classifying a randomly selected set of instances and calculating the proportion assigned to the same class by both trees. The second family \citep{zimmermann2008ensemble,briand2009similarity}, ``structural stability,'' examines the similarity between structural properties of two trees, e.g., by looking at the variance in the size and depth of the trees during cross-validation. Hybrid approaches \citep{dwyer2007decision,wang2018region} rely on region stability or compatibility and estimate the probability that the trees classify a randomly selected example in ``equivalent'' decision regions. None of the above approaches directly compares the two trees' structures; this would be possible using syntactic distance measures \citep{levenshtein1966binary,zhang1989simple,priel2022vectorial}, such as the edit distance, which, however, heavily depend on the representation and consider logically equivalent trees as different \citep{turney1995bias,miglio2004comparison}. We encode decision trees in a way that enables us to overcome this limitation and compute \textbf{trees' structural distance} in the most direct way --- by finding the optimal matching of the trees' paths.


\subsection{Interpretability and its Interface with Stability} \label{ss:interpretable}

\paragraph{Importance of interpretability in ML.} The practical implications of interpretability on the adoption of ML models have been highlighted by numerous recent studies. Practitioners are more likely to use algorithms if they understand and are able to modify them \citep{dietvorst2018overcoming}; this can be particularly beneficial when algorithmic decisions lack domain knowledge and suffer by model misspecification \citep{chen2022algorithmic}, or when human decision makers have access to private information that is unused by the algorithm \citep{ibrahim2021eliciting, balakrishnan2022improving}. Vice versa, interpretable ML can assist and affect human decisions \citep{gillis2021fairness}, or even help improve workers' performance by inferring tips and strategies from the model \citep{bastani2021learning}. Especially in health care applications, the incorporation of ML into clinical medicine raises numerous ethical challenges \citep{char2018implementing} and it is crucial for practitioners to be able to understand the reasoning behind ML models' decisions. This raises questions regarding quantitatively assessing the effect of interpretability on model performance (see \cite{bertsimas2021pricing} for such a study in the context of algorithmic insurance).


\paragraph{Measuring interpretability.} The preceding discussion emphasizes the need to further understand interpretable ML models. \cite{bertsimas2019price} introduce a mathematical framework to rigorously measure a model's interpretability, which, until recently, remained only loosely defined: using their framework, models that are in principle interpretable (including decision trees) can in practice have varying degrees of interpretability. Additionally, Example~\ref{ex:healthcare} raises the question of how interpretable can an unstable model be: a model which changes vastly when the data changes slightly cannot provide trustworthy knowledge concerning the underlying problem. In this paper, we explore the \textbf{stability-interpretability relationship} for decision tree models, which are inherently interpretable, and uncover the interpretability characteristics (e.g., number of nodes) of the most stable trees.

\subsection{Contributions, Outline, and Methodology}

\paragraph{Contributions.} We now summarize the contributions of our work. 
\begin{itemize}
    \item We introduce a novel distance metric for decision tree models. The proposed metric enables us to quantify how structurally different two decision trees are and determine their relative stability. 
    \item We propose a new methodology to train more stable decision tree models. The proposed methodology is particularly relevant in settings where more data is expected to become available over time, in which case the underlying ML model may need to be retrained. 
    \item We demonstrate the value of the proposed methodology through a variety of real-world case studies in health care, where stability and interpretability are both essential, ranging from predicting the risk of deep vein thrombosis to examining the effect of radiotherapy on reducing the risk of local recurrence to patients with sarcoma tumor.
\end{itemize}

\paragraph{Outline.} We organize the rest of the paper as follows.  In Section \ref{s:distance}, we formalize the decision tree problem and introduce the proposed distance metric for decision trees, which plays a central role in our notion of stability. Then, Section \ref{s:stable} presents and evaluates our methodology to train stable decision trees. Finally, in Section \ref{s:case-studies}, we provide a detailed description and qualitative analysis of the six real-world case studies from the health care space which we use to empirically study the proposed methodology.

\paragraph{Experimental methodology and software.} All numerical experiments in this paper rely on six real-world case studies, all of which come from applications in health care. In summary, the case studies we consider are: \textit{Thrombosis}, where we predict the risk of deep vein thrombosis after endovenous thermal ablation; \textit{Sarcoma tumor,} where we examine the effect of radiotherapy on reducing local recurrence within five years to patients with sarcoma tumor; \textit{REBOA}, where we study whether, using ML, we can decrease the misuse of resuscitative endovascular balloon occlusion of the aorta in hemodynamically unstable blunt trauma patients; \textit{TAVR}, where we investigate whether using the appropriate valve type in a transcatheter aortic valve replacement procedure can reduce the need for pacemaker; \textit{Splenic injury}, where we explore how different treatments affect mortality of victims of blunt splenic injury; and \textit{Breast cancer}, where we predict whether a breast cancer is benign or malignant using features computed from a digitized image of a fine needle aspirate of a breast mass. We provide specific details about and a thorough qualitative analysis of the case studies in Section~\ref{s:case-studies}.

In all case studies, we split the data into training (67\%) and testing (33\%) sets, which we use to train and evaluate (out-of-sample), respectively, our models. We repeat the data spliting process multiple times (10, unless otherwise specified), and report the mean and standard deviation of the results. We implement all algorithms in \verb|Python| programming language (version 3.7). We use the \verb|Scikit-learn| implementation \citep{scikit-learn} of the CART \citep{breiman1984classification} and Random Forest \citep{breiman2001random} algorithms. We solve the optimization models using the \verb|Gurobi| commercial solver (version 9.5). All experiments were performed on a standard Intel(R) Xeon(R) CPU E5-2690 @ 2.90GHz running CentOS release 7. 

\section{Measuring the Distance between Decision Trees} \label{s:distance}

Central to our approach is a distance metric between decision trees, which directly compares the trees' structures and, to our knowledge, has been missing from the ML literature; the subject of this section is the definition and study of such a metric. 

\subsection{Decision Tree Problem Definition} \label{ss:decision-tree-problem-definition}

We start by introducing the decision tree problem and the notation we use throughout the paper.

\paragraph{Notation.}
We are given data $\bm X \in \mathbb{R}^{N\times P}$ and responses $\bm y \in \mathbb{R}^N$; noting that our work naturally generalizes to regression problems with $\bm y \in \mathbb{R}^N$, we focus, throughout this paper, on classification problems with classes $[K] := \{1,\dots,K\}$ so that $\bm y \in [K]^N$. We denote by $\mathcal{N} \subseteq [P]$ the set of numerical features and $\mathcal{C} = [P] \setminus \mathcal{N}$ the set of categorical features. Each numerical feature is characterized by its upper bound $u_j, j\in \mathcal{N}$, and its lower bound $l_j,  j\in \mathcal{N}$; each categorical feature is characterized by its number of categories $c_j,  j\in \mathcal{C}$. Therefore, a problem's feature space is defined by a collection of three vectors $(\bm u, \bm l, \bm c) \in \mathbb{R}^{|\mathcal{N}|} \times \mathbb{R}^{|\mathcal{N}|} \times \mathbb{N}^{|\mathcal{C}|}$ corresponding, respectively, to the numerical features' upper and lower bounds, and the categorical features' number of categories.

\paragraph{Problem statement.}
In their simplest form, decision trees (with parallel splits) partition the feature space into a set of rectangles, and then assign a prediction to each; for classification problems, the assigned prediction is a class label $k \in [K]$. Each data point $\bm x \in \mathbb{R}^P$ is classified according to the class label that corresponds to the rectangle where it lies. The decision tree learning problem can be described as searching for a way to partition the feature space, such that an error metric, e.g., the number of misclassified points in the training data, is minimized --- possibly subject to additional terms and constraints that, respectively, penalize or prevent more granular partitionings or, equivalently, more complex trees.

\subsection{Decision Tree Representation} \label{ss:decision-tree-representation}

We now present a compact representation of a decision tree as a collection of paths, which facilitates the measurement of the distance between trees.

\paragraph{Definition of a split and feature space partitioning.}
A trained decision tree performs a sequence of ``splits'' on a subset of features. To split on a numerical feature $j \in \mathcal{N}$, we test whether its value is below a threshold $t\in [l_j,u_j]$. To split on a categorical feature $j \in \mathcal{C}$, we test its membership in a subset of categories $\mathcal{C}' \subseteq [c_j]$.

A sequence of splits partitions the feature space in the following way. Starting at the root node, the tree performs a split on a feature and therefore partitions the feature space into two disjoint rectangles. In the rectangle where the condition is satisfied, $t$ defines an upper bound for $j \in \mathcal{N}$ or $\mathcal{C}'$ determines the set of qualifying categories for $j \in \mathcal{C}$. In the rectangle where the condition is not satisfied, $t$ defines a lower bound for $j \in \mathcal{N}$ or $(\mathcal{C}')^C$ determines the qualifying categories for $j \in \mathcal{C}$. Then, the tree (possibly) further partitions each of the two resulting rectangles, each by splitting on a (possibly different) feature. The same logic is applied to every node in the tree, until a leaf node is reached.

\paragraph{Definition of a tree path.}
After performing a full sequence of splits, we obtain one of the final rectangles. Each such sequence, together with the class label that is assigned to the resulting rectangle define a \emph{tree path}. Path $p$ is characterized by the upper and lower bounds that are imposed on numerical features, the qualifying categories for categorical features, and the assigned class label. We represent the qualifying categories for feature $j$ in path $p$ as binary vector $\bm c_j^p \in \{0,1\}^{c_j}$ with ones in positions that correspond to categories that qualify across path $p$. We then append zeros to all such vectors (so they are of the same length) and stack them to form matrix $\bm C^p \in \{0,1\}^{|\mathcal{C}| \times \max_j c_j}$. Put together, we represent a tree path as $(\bm u^p, \bm l^p, \bm C^p, k^p) \in \mathbb{R}^{|\mathcal{N}|} \times \mathbb{R}^{|\mathcal{N}|} \times \{0,1\}^{|\mathcal{C}| \times \max_j c_j} \times [K]$.

\paragraph{Representation of a tree.}
A tree $\mathbb{T}$ is then (non-uniquely) represented as a collection of $T$ paths $\mathcal{P}(\mathbb{T}) = \{p_1, \dots, p_{T}\}$, where $|\mathcal{P}(\mathbb{T})|=T$. The non-uniqueness of this representation of decision trees owes to the fact that the order in which splits are performed does not matter. Thus, multiple trees can result in the same collection of paths $\mathcal{P}$. We believe this is a desirable property, since two trees that result in the same collection of paths decide on which class label to assign to any data point by testing the exact same set of conditions (albeit in different order).

\subsection{Distance between Paths} \label{ss:distance-between-paths}

In this section, we define two quantities that serve as building blocks in measuring the distance between two trees: the distance between two paths and the notion of a path's weight.

\paragraph{Paths' distance.}
To measure the distance between two paths, we compare the feature ranges and the class label that each path results in. Intuitively, two paths are close if they result in overlapping rectangles. This leads to the following definition for the distance between paths $p$ and $q$: 
\begin{equation} \label{eqn:dist-paths}
    d(p,q) = \sum_{j \in \mathcal{N}} \frac{\left|u^p_j - u^q_j\right| + \left|l^p_j - l^q_j\right|}{2(u_j - l_j)} + \sum_{j \in \mathcal{C}} \frac{\left\|\bm c^p_j-\bm c^q_j\right\|_1}{c_j} + \lambda \cdot \mathbbm{1}_{(k^p \neq k^q)},
\end{equation}
where $\mathbbm{1}_{(\cdot)}$ denotes the indicator function. We weigh the last term in $d(p,q)$ by $\lambda$ to adjust the relative importance in comparing the feature ranges and the class labels between the two paths. For example, two paths of depth $D$ can result in different feature ranges for at most $2D$ features. By setting $\lambda=2D$, we assign equal weight to the amount of overlap in feature ranges and the resulting class labels between the two paths.
We remark that instead of comparing the paths' class labels, we can compare the leaf class distributions by simply replacing the indicator $\mathbbm{1}_{(k^p \neq k^q)}$ in $d(p,q)$ with, e.g., the leaf node's Gini impurity. By doing so, the resulting paths' distance would incorporate statistical considerations too; as we are interested in structural stability, we do not address this here.

\paragraph{Path weight.}
In addition, to quantify a path's importance, we introduce the notion of ``path weight,'' which captures the portion of the feature ranges that the path covers (among features that are used in the path's split nodes) and is defined as follows:
\begin{equation} \label{eqn:weight-path}
    w(p) = \sum_{j \in \mathcal{N}} \frac{u^p_j - l^p_j}{u_j - l_j} \cdot \mathbbm{1}_{\left(u^p_j \neq u_j \text{ or } l^p_j \neq l_j\right)} + \sum_{j \in \mathcal{C}} \frac{c^p_j}{c_j}\cdot \mathbbm{1}_{\left(c^p_j \neq c_j\right)}.
\end{equation} 
Intuitively, a path that is assigned a heavy weight will lead to rectangles which include large portions of the ranges of the numerical features or many categories for the categorical features. The path weight will be used to measure the distance between trees with different numbers of paths. 

Notice that, in measuring both the paths' distance and the path weight, each feature is divided by its range or number of categories, so both quantities are expressed in the same scale.

\subsection{Distance between Trees and Computation} \label{ss:distance-between-trees}

Using the decision tree representation of Section \ref{ss:decision-tree-representation} and the path-related quantities of Section \ref{ss:distance-between-paths}, we are ready to introduce the proposed distance metric for decision trees.

\paragraph{Trees' distance.} To measure the distance between two trees, we look at how different their paths are in terms of the features each path splits on, the split thresholds or categories (for numerical or categorical features respectively), and the resulting class label. Such an approach captures structural differences between the two trees, instead of, e.g., comparing the distributions of outcomes. Intuitively, trees that are close will consist of similar paths and therefore lead to similar partitionings of the feature space. The proposed distance metric compares the two trees' paths and searches for a way to optimally match them.

Formally, we are interested in measuring the distance between trees $\mathbbm{T}_1$ and $\mathbbm{T}_2$. We assume, without loss of generality, that $T_1 > T_2$, that is, $\mathbbm{T}_1$ consists of a larger number of paths. We introduce decision variables $x_{pq} = \mathbbm{1}(\text{path $p$ in $\mathbbm{T}_1$ is matched with path $q$ in $\mathbbm{T}_2$})$ and $x_{p} = \mathbbm{1}(\text{path $p$ in $\mathbbm{T}_1$ is left unmatched})$. We formulate the following integer linear optimization problem:
\begin{equation} \label{eq:distance}
    \begin{split}
        d(\mathbbm{T}_1, \mathbbm{T}_2)
        \ =\ \min_{\bm x} & \quad \sum_{p \in \mathcal{P}(\mathbb{T}_1)} \sum_{q \in \mathcal{P}(\mathbb{T}_2)} d(p,q) x_{pq} 
        + \sum_{p \in \mathcal{P}(\mathbb{T}_1)} w(p) x_{p}\\
        \text{s.t.} & \quad \sum_{q \in \mathcal{P}(\mathbb{T}_2)} x_{pq} + x_{p} = 1, \quad \forall p \in \mathcal{P}(\mathbb{T}_1) \\
        & \quad \sum_{p \in \mathcal{P}(\mathbb{T}_1)} x_{pq} = 1, \quad \forall q \in \mathcal{P}(\mathbb{T}_2) \\
        & \quad x_{pq} \in \{0,1\}, \quad x_{p} \in \{0,1\}, \qquad \forall p \in \mathcal{P}(\mathbb{T}_1), \quad \forall q \in \mathcal{P}(\mathbb{T}_2)
    \end{split}
\end{equation}
Upon solving Problem~\eqref{eq:distance}, each path $p \in \mathcal{P}(\mathbb{T}_1)$ will be either matched with a path $q \in \mathcal{P}(\mathbb{T}_2)$, in which case the tree distance $d(\mathbbm{T}_1, \mathbbm{T}_2)$ will increase by the distance $d(p,q)$ between the two paths, or will remain unmatched, in which case the tree distance will increase by the path's weight $w(p)$. 
We note that, if $T_1 = T_2$, that is, the two trees consist of the same number of paths, we do not include $x_p$ in the formulation.

The following proposition, which we prove in Appendix~\ref{appx:proofs-metric}, suggests that the proposed distance measure satisfies the requirements of a metric: the distance from a tree to itself is zero, the distance between two distinct trees is positive, the distance is symmetric, and it satisfies the triangle inequality.
\begin{proposition} \label{lem:metric}
Let $\mathcal{T}$ denote the set of all trees of maximum depth $D$ and $\mathbbm{T}_1, \mathbbm{T}_2 \in \mathcal{T}$. 
Then, $d(\mathbbm{T}_1, \mathbbm{T}_2)$ is a metric mapping $\mathcal{T}\times \mathcal{T} \mapsto \mathbb{R}.$
\end{proposition}

\paragraph{Computation.} To compute the distance between two trees, our proposed approach requires solving Problem~\eqref{eq:distance}, which is a variant of bipartite matching and therefore efficiently solvable. In particular, consider the linear relaxation of Problem~\eqref{eq:distance}, where the binary constraints $x_{pq} \in \{0,1\}$ and $ x_{p} \in \{0,1\}$ are replaced with linear constraints $0 \leq x_{pq} \leq 1$ and $0 \leq x_{p} \leq 1$. We have the following corollary, which, for completeness, we prove in Appendix~\ref{appx:proofs-relaxation}:
\begin{corollary} \label{cor:relaxation}
Any extreme point of the linear relaxation of Problem~\eqref{eq:distance} is a binary vector.
\end{corollary}
Corollary~\ref{cor:relaxation} guarantees that the optimum to the linear relaxation of Problem~\ref{eq:distance} is the incidence vector of a perfect matching and hence encodes an optimal matching of paths. Owing to this result, we can compute the distance between trees in polynomial time (and very efficiently in practice) by simply solving a linear optimization problem.

\paragraph{An upper bound on the distance.} To get a relative (and more intuitive) sense of how close two trees are, we properly scale the distance so that it expresses a percentage of the maximal amount two trees of depth $D$ can differ. To do so, we derive a problem-independent upper bound on the proposed distance metric for given maximum allowable tree depth $D$:
\begin{proposition} \label{lem:upperbound}
Given trees $\mathbbm{T}_1$ and $\mathbbm{T}_2$ with $\text{depth}(\mathbbm{T}_1) \leq D$ and $\text{depth}(\mathbbm{T}_2) \leq D$, it holds that $d(\mathbbm{T}_1, \mathbbm{T}_2)~\leq~2^D~(2D+\lambda).$
\end{proposition}
Owing to Proposition~\eqref{lem:upperbound}, which we prove in Appendix~\ref{appx:proofs-upperbound}, upon computing the distance, we scale it by $\nicefrac{1}{2^D (2D+\lambda)}$ and hence get an expression of the distance as a percentage of the distance between the two trees that are as far apart as possible. 





\subsection{Tree Distance in Practice}

We now return to the setting described in Example \ref{ex:healthcare} and study the proposed distance metric through a simple practical example. We use the sarcoma tumor case study (see Section \ref{s:case-studies} for details). 

Using an initial dataset $\bm X_0 \in \mathbb{R}^{N/2\times P}$ (which we build by sampling $\frac{N}{2}$ data points), we train two decision trees. For the first one, shown in the top-left panel of Figure \ref{fig:simple-trees} and referred to as \texttt{CART CV}$_0$, we apply a standard 5-fold cross-validation procedure. For the second one, shown in the bottom-left panel of Figure \ref{fig:simple-trees} and referred to as \texttt{CART Pareto}$_0$, we apply our proposed methodology (described in Section \ref{ss:stable-methodology}). At a later time, when the full dataset $\bm X \in \mathbb{R}^{N\times P}$ becomes available, we retrain the two decision trees by (independently) repeating the aforementioned two methodologies --- using $\bm X$ instead of $\bm X_0$. We refer to the resulting trees as \texttt{CART CV} (top-right panel of Figure \ref{fig:simple-trees}) and \texttt{CART Pareto} (bottom-right panel of Figure \ref{fig:simple-trees}). We note that we tuned all trees using the same set of candidate hyperparameters and, for simplicity in presentation, restricted the maximum depth to 4.

\begin{figure}[!ht]
\begin{subfigure}{0.48\textwidth}
\includegraphics[width=\linewidth]{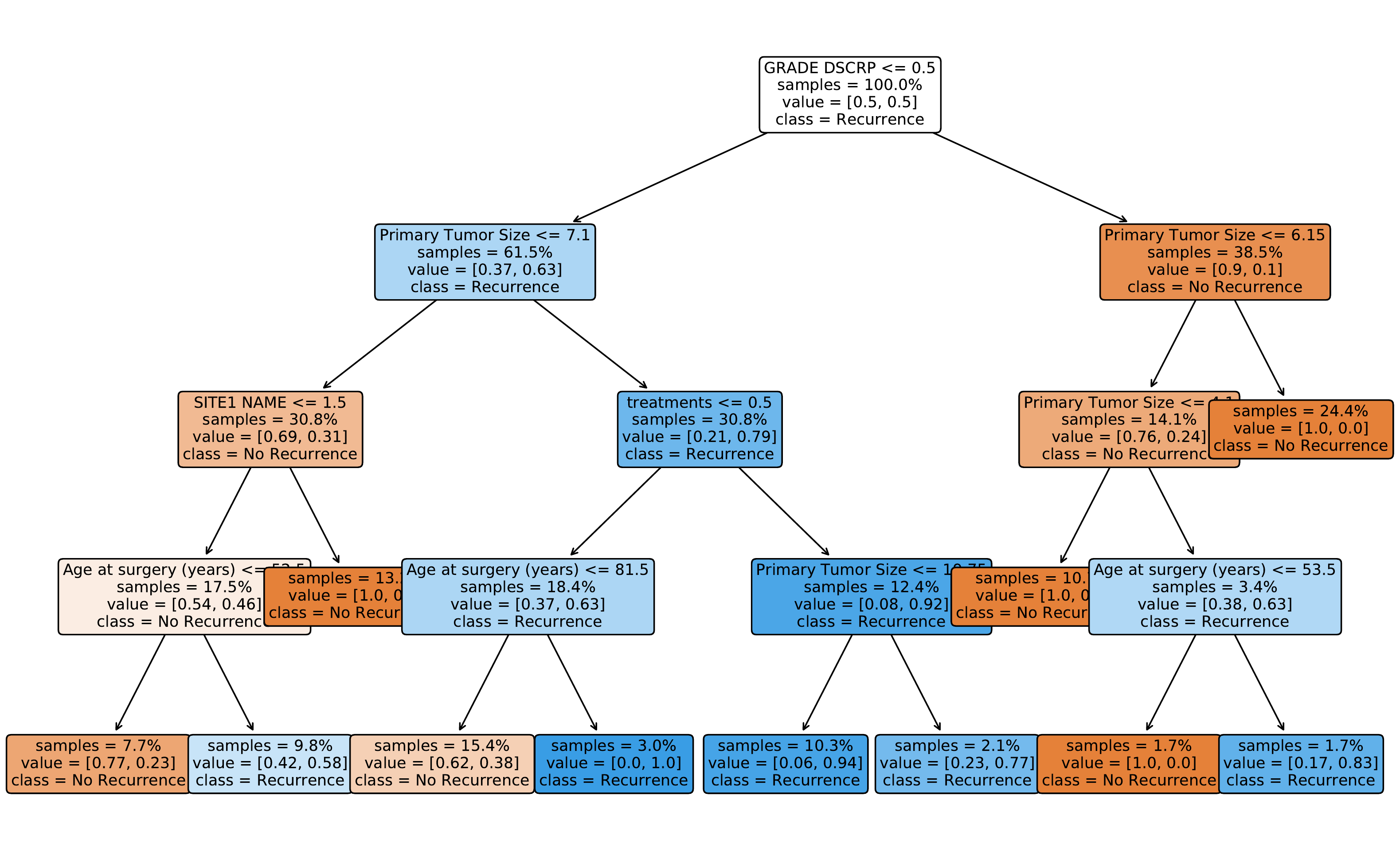} 
\end{subfigure}\hspace*{\fill}
\begin{subfigure}{0.48\textwidth}
\includegraphics[width=\linewidth]{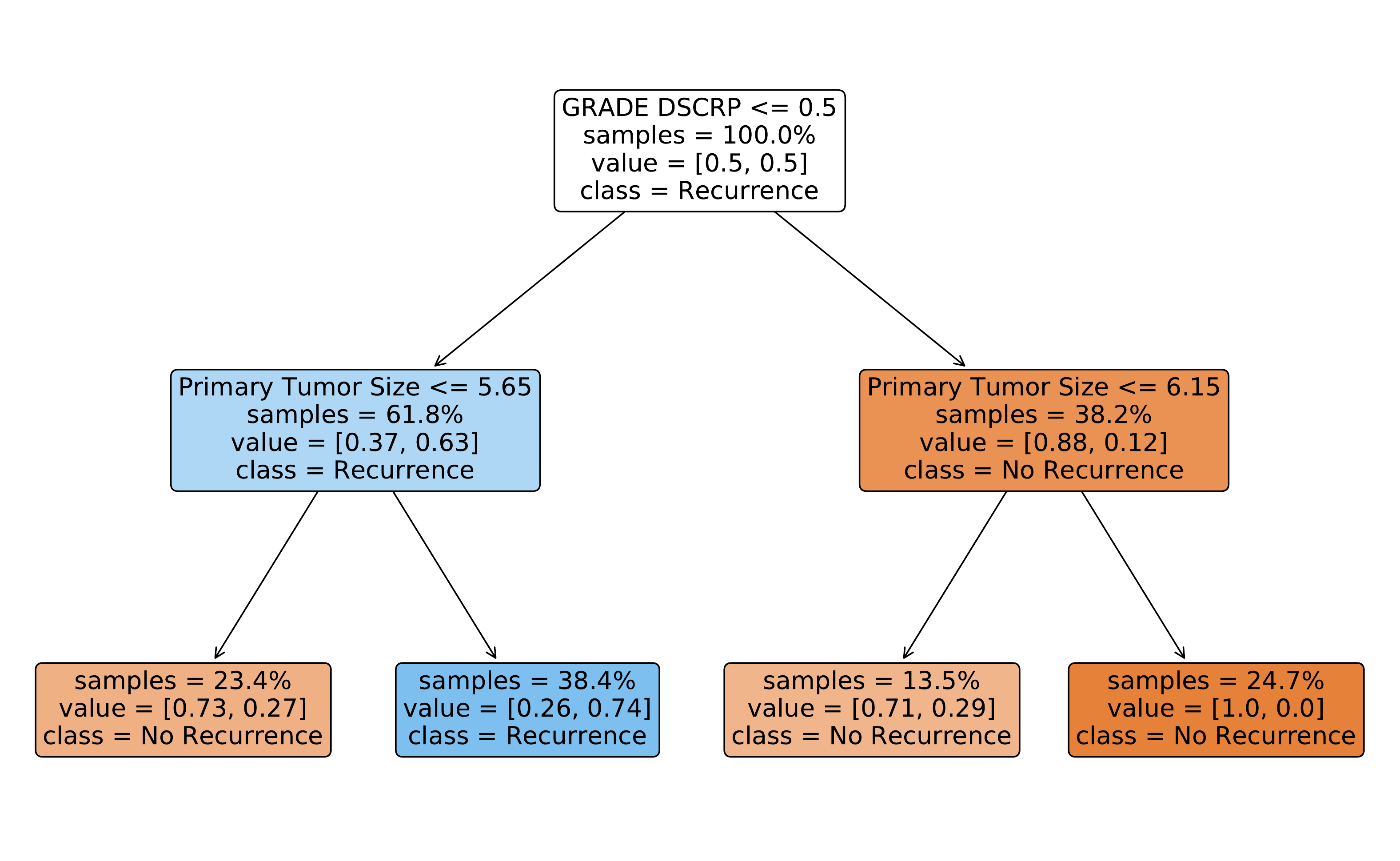} 
\end{subfigure}
\begin{subfigure}{0.48\textwidth}
\includegraphics[width=\linewidth]{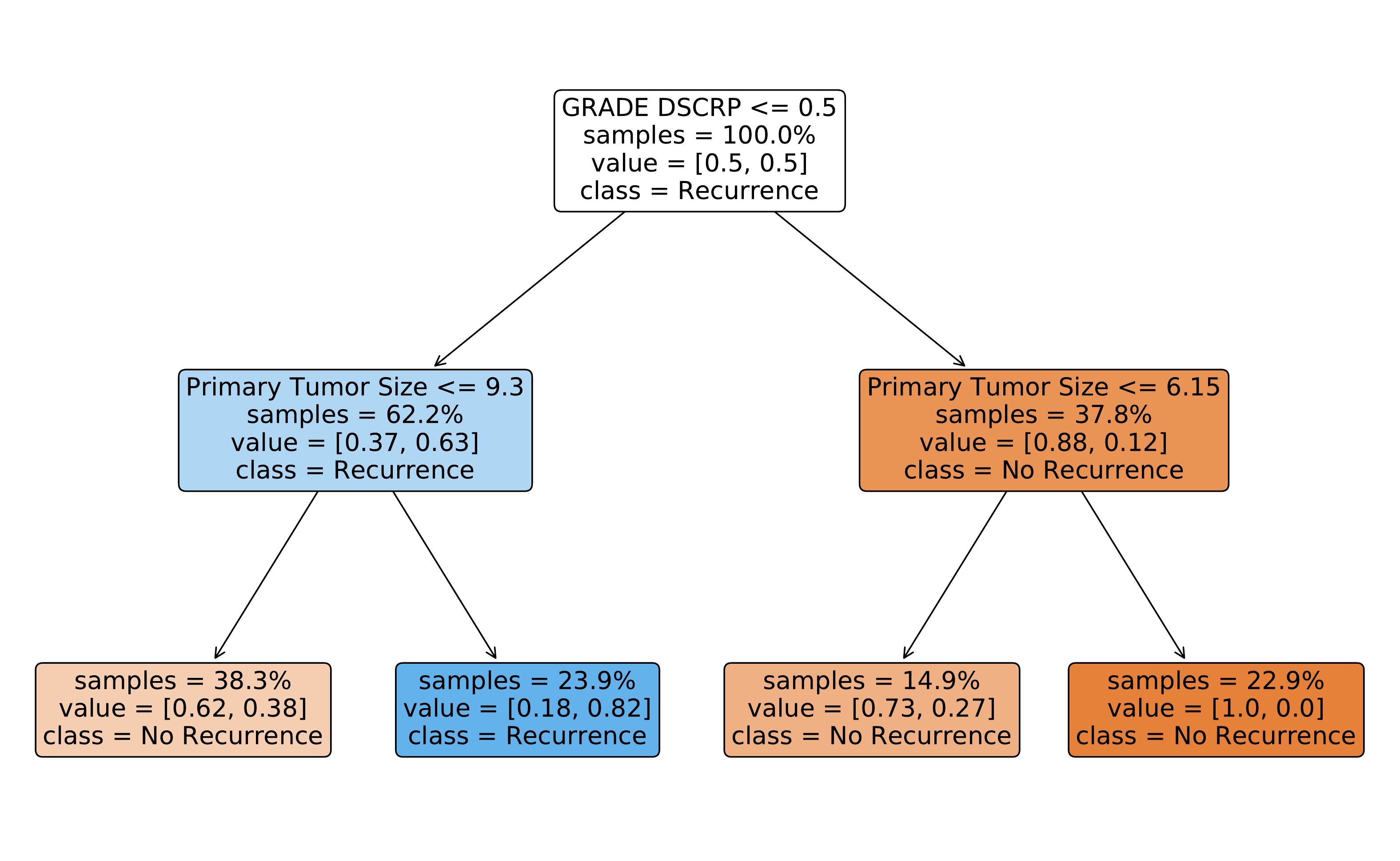} 
\end{subfigure}\hspace*{\fill}
\begin{subfigure}{0.48\textwidth}
\includegraphics[width=\linewidth]{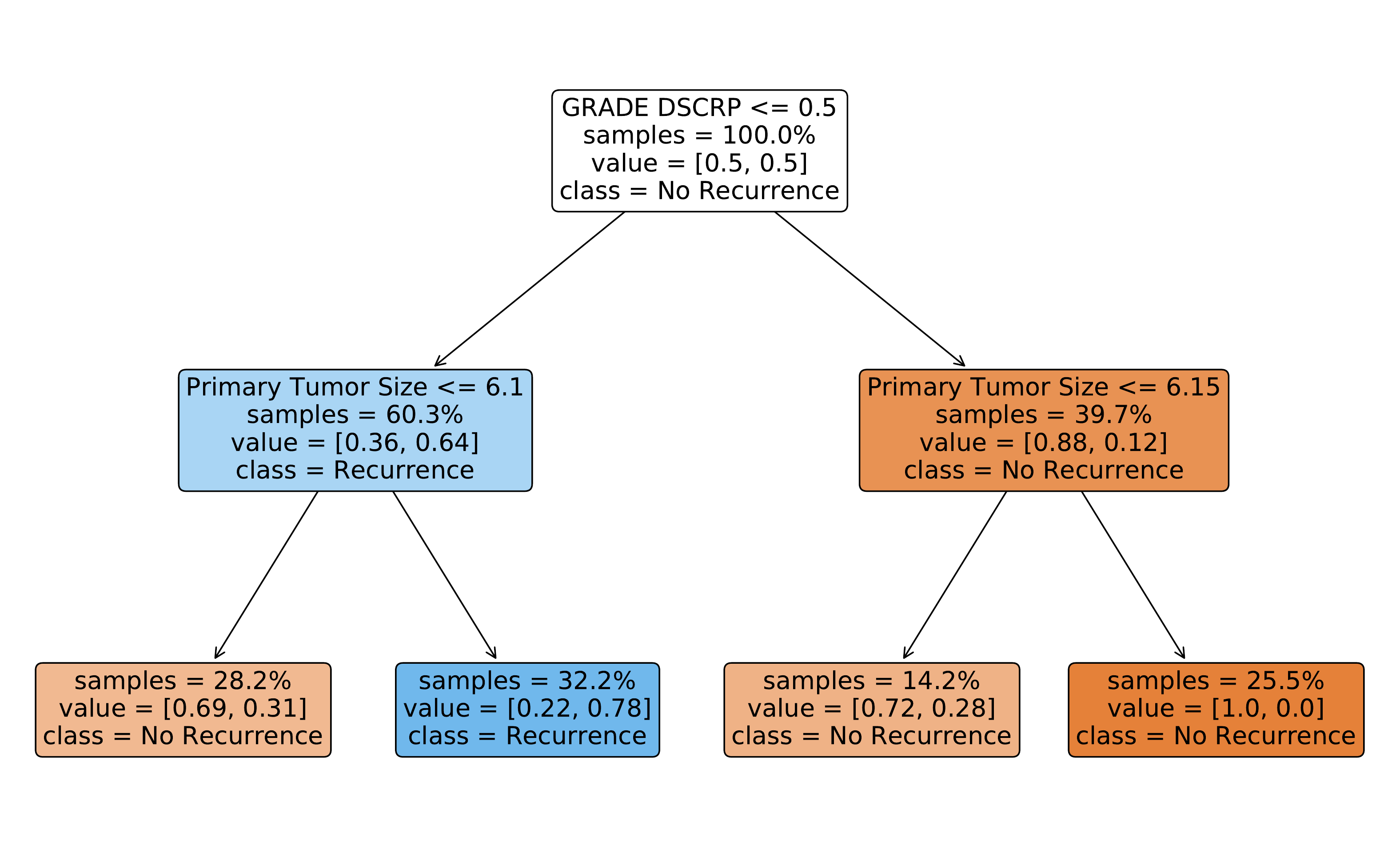} 
\end{subfigure}
\caption{Trees obtained using cross-validation (top) versus trees obtained using the proposed stable methodology of Section \ref{ss:stable-methodology} (bottom). For the left column trees, we used an initial training set of $\frac{N}{2}$ data points; for the right column trees, we used the full training set of $N$ data points. } \label{fig:simple-trees}
\end{figure}

Clearly, the trees obtained using cross-validation are structurally much more different from each other compared to the trees obtained using the proposed stable methodology. \texttt{CART CV}$_0$ and \texttt{CART CV} differ in both their depth (4 versus 2) and in the number of distinct features they split on (5 versus 2). In contrast, \texttt{CART Pareto}$_0$ and \texttt{CART Pareto} split on the exact same features and the only difference is a slight shift in the left child node's threshold by about $7\%$ of the corresponding feature range. The proposed distance metric captures this visual observation: $d(\texttt{CART CV}_0, \texttt{CART CV}) = 11.21 \%$ and $d(\texttt{CART Pareto}_0, \texttt{CART Pareto}) = 0.18 \%$, 
where the distances are expressed as a percentage of the distance between the two trees of depth $4$ that are as far apart as possible.  Closing our Example \ref{ex:healthcare} analogy, we claim that a physician who had been using \texttt{CART Pareto}$_0$ will likely trust \texttt{CART Pareto} upon retraining; it would be more difficult for them to transition from \texttt{CART CV}$_0$ to \texttt{CART CV}.


\subsection{Sensitivity Analysis of the Proposed Tree Distance} \label{ss:numerical-experiments-distance}

Before concluding the section, we discuss some key numerical properties of the proposed tree distance metric. The corresponding numerical study, which uses the case studies of Section \ref{s:case-studies}, is given in Appendix \ref{appx:numerical-experiments-distance}. We test the sensitivity of the proposed distance metric to two types of perturbations --- direct and indirect ones. Direct perturbations refer to immediate interventions and changes in the tree structure. Indirect perturbations refer to modifications in the training data. The main takeaways are the following:
\begin{itemize}
    \item The proposed distance metric has desirable properties, including having a linearly increasing relationship with the amount of direct perturbation in the tree and showing an increasing trend with the amount of indirect perturbation in the tree (i.e., the fraction of the training data that changes).
    \item The upper bound derived in Section \ref{ss:distance-between-trees} is conservative in that, trees where the thresholds are allowed to vary by 50\% in expectation and 100\% at maximum have an expected distance of 12.5\% and a standard deviation in the distance of 5\%, expressed as percentage of the maximum possible distance between any two trees of the given depth.
\end{itemize}

\section{Computing Stable Decision Trees} \label{s:stable}

In this section, we develop and empirically evaluate a methodology that improves the stability of decision tree models, using the distance metric we introduced in Section \ref{s:distance}. 

\subsection{Training Collections of Stable Trees} \label{ss:stable-methodology}

We now describe the proposed methodology to train collections of stable decision trees. The methodology is motivated by the health care setting of Example \ref{ex:healthcare}. We have an initial patient database (corresponding to, e.g., any of the case studies we describe in Section \ref{s:case-studies}) and train a model using the data available at the time. Later, when more patient data becomes available, we retrain the model to improve its accuracy (owing to the larger sample size) and to capture new patterns that may be present in the data. \emph{It is crucial that the new model does not deviate too much from the previous one, as this could affect physicians' trust and willingness to use the model.} We address such concerns by incorporating a notion of stability in the tree training and selection process.

More concretely, the proposed methodology consists of the following steps:
\begin{itemize}
    \item We randomly split the training data $\bm X$ into two batches, $\bm X_0 \in \mathbb{R}^{N_0\times P}$ and $\bm X_1 \in \mathbb{R}^{N_1\times P}$, such that $N_0+N_1=N$. 
    
    \item Using the first batch of training data, $\bm X_0$, we train a first collection of $B$ trees $\mathcal{T}_0 = \{ \mathbb{T}_1^0, \dots, \mathbb{T}_B^0\}$. We obtain $\mathcal{T}_0$ by bootstrapping $\bm X_0$, i.e., generating multiple new training datasets by sampling $\bm X_0$ uniformly and with replacement. For each dataset, we train decision trees with different hyperparameters: tree depth $D \in \{3,\dots,12\}$ and minimum number of samples per leaf $M \in \{3,5,10,30,50\}$. 
    
    \item Using the full training data, $\bm X$ (hence merging the two batches $\bm X_0$ and $\bm X_1$), we train a second collection of $B$ trees $\mathcal{T} = \{ \mathbb{T}_1, \dots, \mathbb{T}_B\}$, in the exact same way: we bootstrap $\bm X$ and examine, for each dataset, a large set of hyperparameters.

    \item We then compute, for every tree in the second batch $\mathbb{T}_b \in \mathcal{T}, b \in [B],$ its mean distance from all trees in the first batch $$d_b = \sum_{\beta=1}^B \frac{d(\mathbb{T}_\beta^0,\mathbb{T}_b)}{B}.$$
    In addition, using holdout (validation or testing) data $(\bm X_{\text{holdout}}, \bm y_{\text{holdout}})$, we estimate, for every tree in the second batch $\mathbb{T}_b \in \mathcal{T}, b \in [B],$ its out-of-sample predictive performance. For example, in binary classification problems, we use the area under the ROC curve $$\alpha_b = \text{AUC}(\mathbb{T}_b; \bm X_{\text{holdout}}, \bm y_{\text{holdout}}).$$

    \item We obtain the collection $\{(\mathbb{T}_b,d_b,\alpha_b)\}_{b=1}^B$, where each tree $\mathbb{T}_b \in \mathcal{T}$ is characterized by two, possibly competing, metrics: $d_b$, characterizing its stability, and $\alpha_b$, characterizing its predictive power. For any tree $\mathbb{T}_b$, we need to guarantee that there exists no another tree that performs at least as well on one metric and strictly better on the other; if such a tree existed, it would be strictly preferred for all practical considerations. To address this situation, we search for the Pareto frontier, that is, the set of Pareto optimal trees. A tree $\mathbb{T}_b \in \mathcal{T}$ is Pareto optimal if there exists no other tree $\mathbb{T}_{b'} \in \mathcal{T}, b\not=b',$ with $$(d_{b'} \leq d_b \text{ and } \alpha_{b'} > \alpha_b) \text{ or } (d_{b'} < d_b \text{ and } \alpha_{b'} \geq \alpha_b).$$ We denote the set of Pareto optimal trees (satisfying the aforementioned condition) by $\mathcal{T}^\star \subseteq \mathcal{T}.$ 

    \item Once $\mathcal{T}^\star$ is computed, we can select a tree from the Pareto frontier using an application-specific function:
    \begin{equation} \label{eq:select-tree}
        \mathbb{T}^\star = \underset{\mathbb{T}_b \in \mathcal{T}^\star}{\text{argmax}} \ f(d_b, \alpha_b).
    \end{equation}
    For example, for a given suboptimality tolerance $\epsilon$, we can use: $f(d_b, \alpha_b) =  (1 - d_b) \cdot \mathbbm{1}~{\left(\alpha_b \geq (1-\epsilon)\underset{b'}{\max}\ \alpha_{b'} \right)}$.
    Alternatively, if we are equally interested in stability and predictive power, we can choose the tree $\mathbb{T}^\star$ that maximizes $f(d_b, \alpha_b) =  \frac{- d_b + \alpha_b}{2}.$
\end{itemize}

The proposed methodology relies on the assumption that decision trees' stability is negatively correlated with the proposed tree distance metric. By selecting the tree in the second batch that is closest, on average, to the first batch trees, we are choosing a ``centroid'' tree, which is similar to many among a large number of trees obtained using a large number of datasets and parameters. Intuitively, such a tree is likely to be a stable one.

As a final remark, we note that it is possible to characterize the trained trees in terms of additional metrics. For example, we may be willing to assign an interpretability score $i_b$ (encoding, e.g., the number of nodes in the tree) to each tree. In that case, we would need to identify Pareto optimal trees in three dimensions and extract the set $\mathcal{T}^\star$ using the collection $\{(\mathbb{T}_b,d_b,\alpha_b,i_b)\}_{b=1}^B$. The output tree would then be selected using a function $f(d_b,\alpha_b,i_b)$ of all metrics of interest.

\subsection{Pareto Optimal Trees} \label{ss:pareto}

We now numerically test our hypothesis on the existence of a set of Pareto optimal trees using the case studies of Section \ref{s:case-studies}. For each case study, we train a collection of trees $\mathcal{T}$ using the methodology of Section \ref{ss:stable-methodology} and, for each tree $\mathbb{T}_b$, we calculate its out-of-sample AUC $\alpha_b$ and mean distance from the first batch $d_b$. We plot the results in Figure~\ref{fig:frontiers}. In red, we show the set $\mathcal{T}^\star$ of Pareto optimal trees; in blue, we show the set of Pareto dominated trees. The formation of a Pareto frontier is clearer in the Thrombosis, Sarcoma tumor, TAVR, and Breast cancer case studies; in the remaining case studies, the set of Pareto optimal trees is concentrated in the bottom right corner of the graph --- in which case the stability-predictive power trade-off is less profound, and the selection of the final tree should be easier. In Figure~\ref{fig:bar-frontier}, we give summary statistics about the size of the Pareto frontier for each case study and across 10 independent repetitions of the experiment (training-testing data splits). At maximum, the frontier never exceeds nine trees; on average, it consists of only five trees; in one case, there exists just one Pareto optimal tree.

\begin{figure}[!ht]
\begin{subfigure}{0.48\textwidth}
\includegraphics[width=\linewidth]{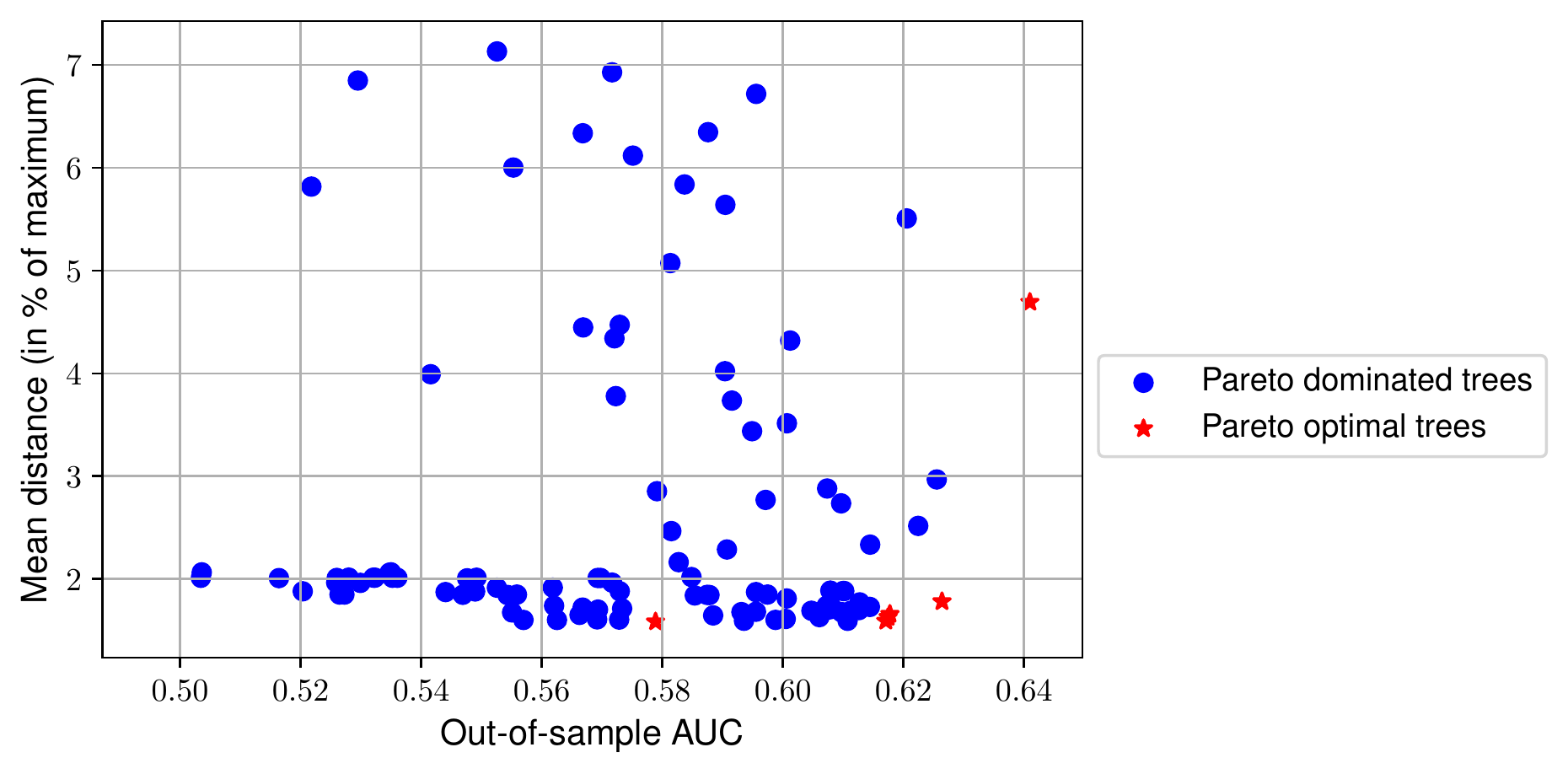}
\caption{Thrombosis} \label{fig:frontier-Thrombosis}
\end{subfigure}\hspace*{\fill}
\begin{subfigure}{0.48\textwidth}
\includegraphics[width=\linewidth]{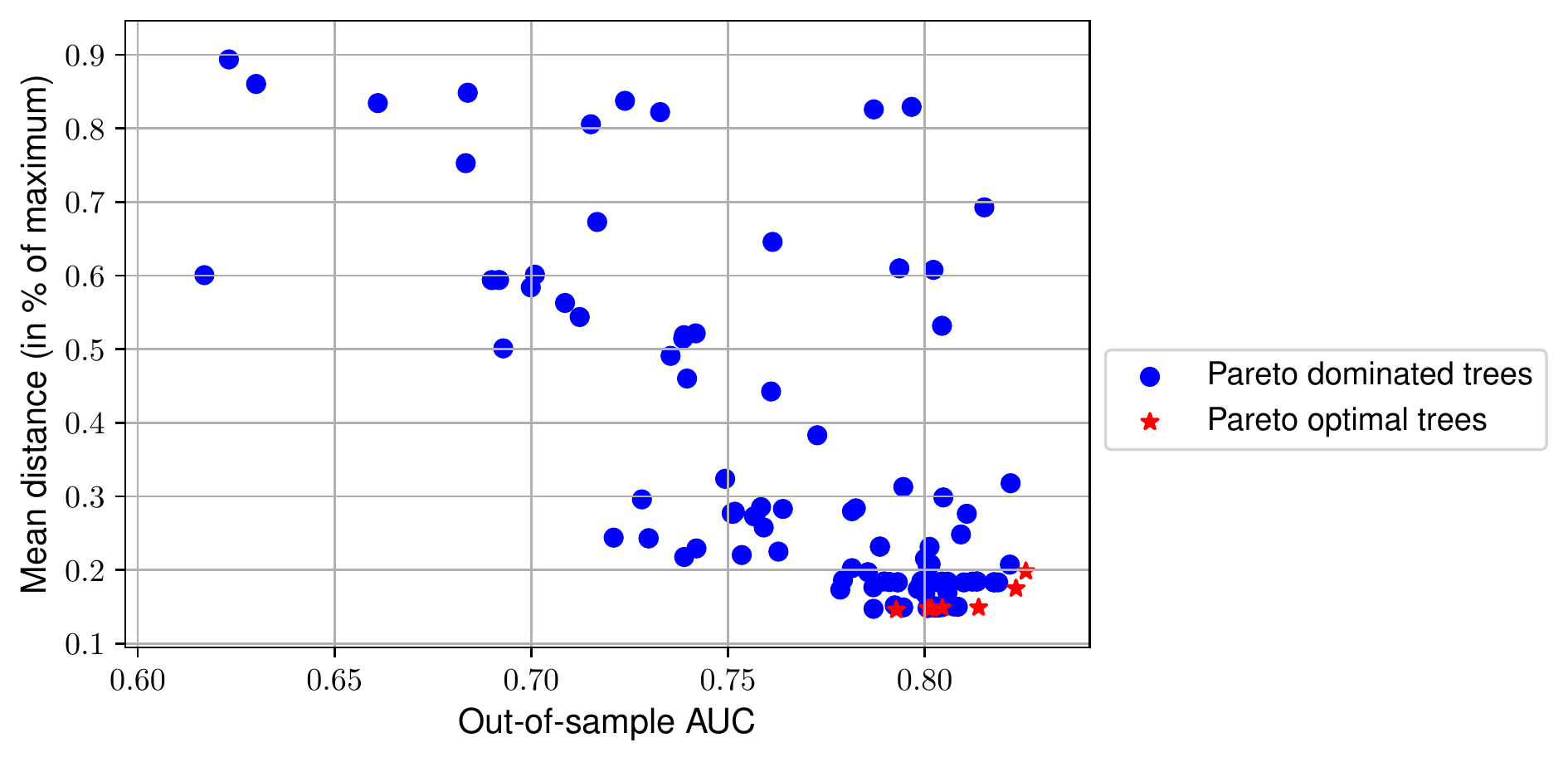}
\caption{Sarcoma Tumor} \label{fig:frontier-sarcoma-tumor}
\end{subfigure}

\medskip
\begin{subfigure}{0.48\textwidth}
\includegraphics[width=\linewidth]{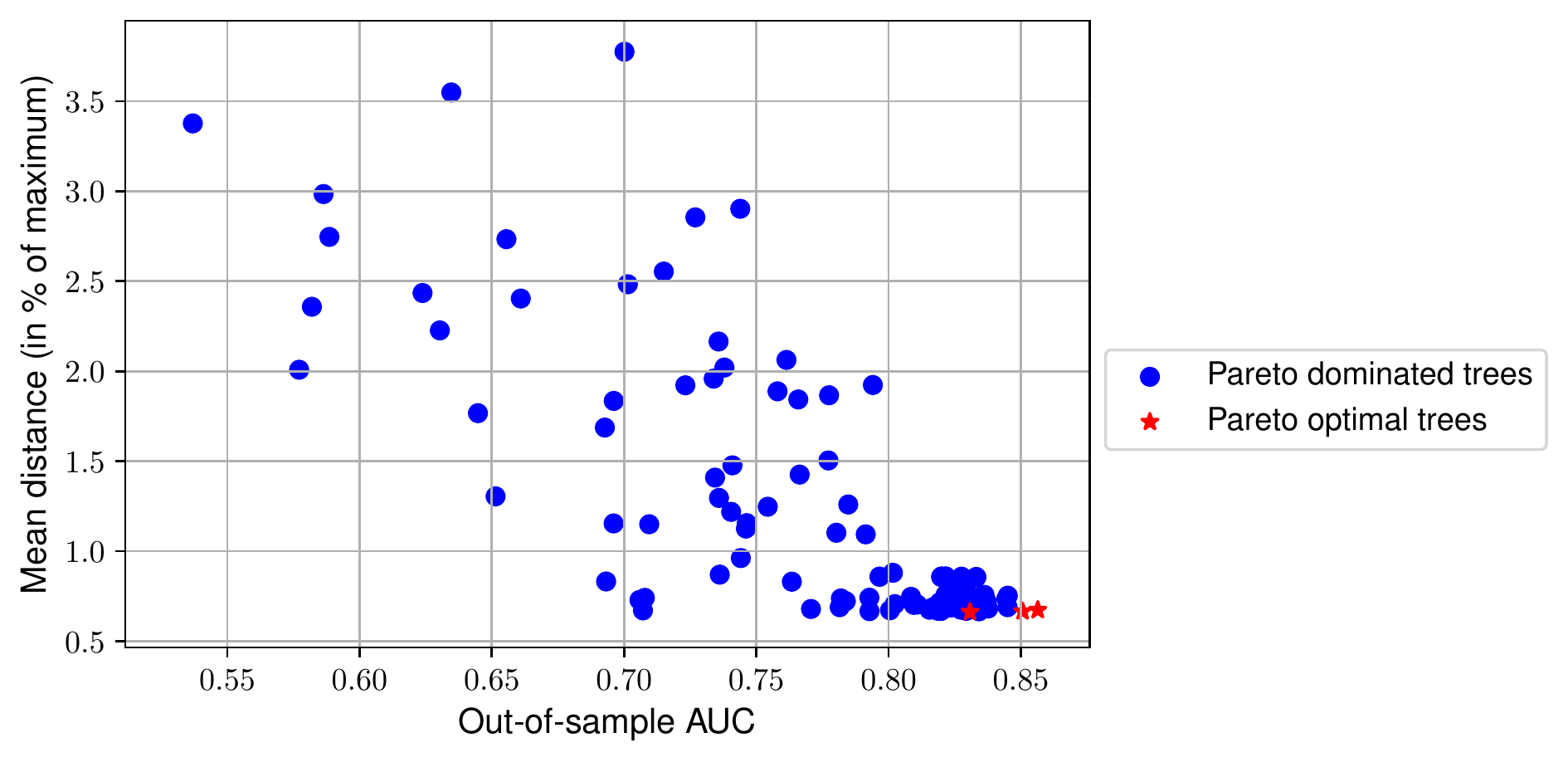}
\caption{REBOA} \label{fig:frontier-REBOA}
\end{subfigure}\hspace*{\fill}
\begin{subfigure}{0.48\textwidth}
\includegraphics[width=\linewidth]{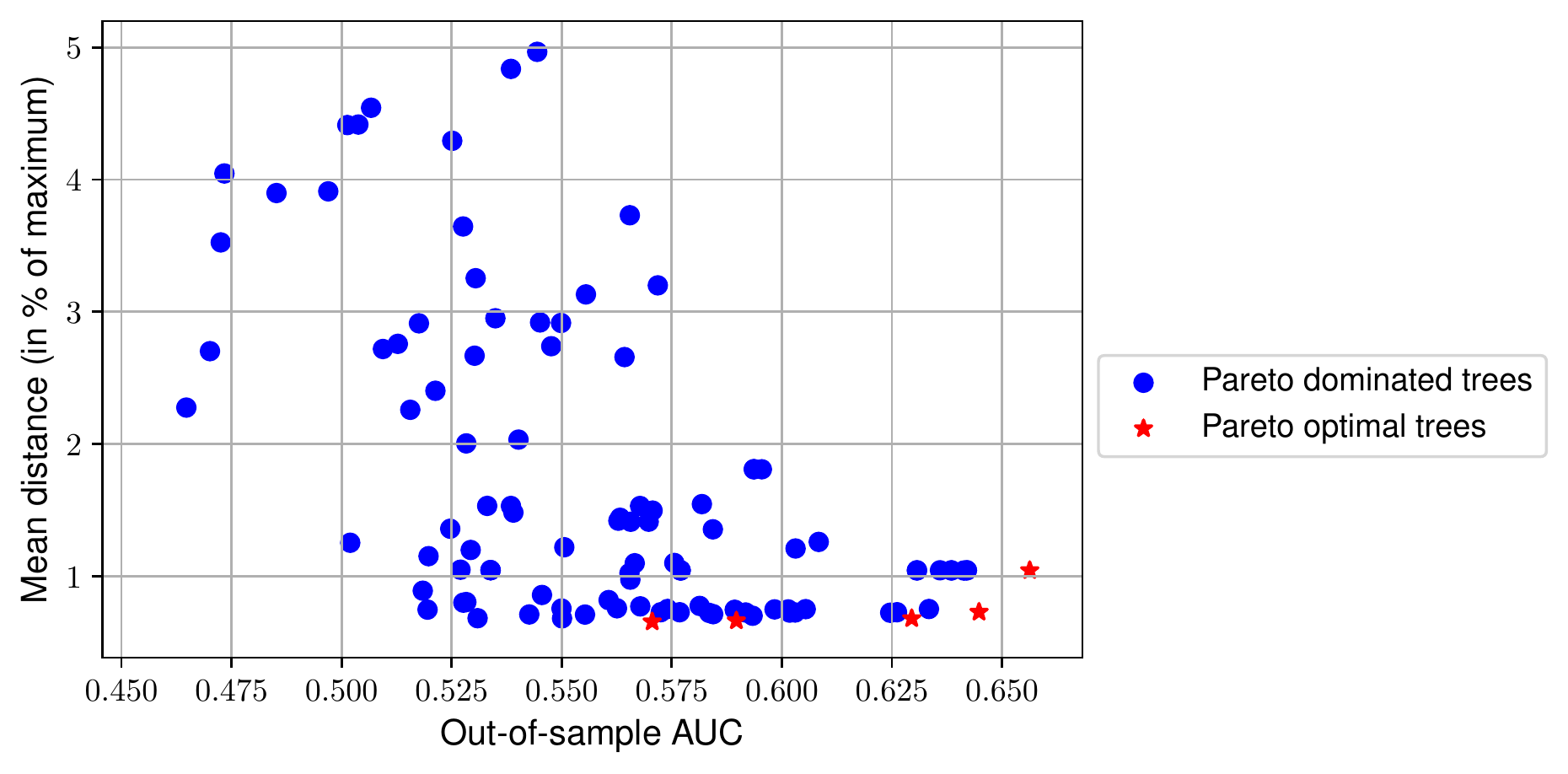}
\caption{TAVR} \label{fig:frontier-TAVR}
\end{subfigure}

\medskip
\begin{subfigure}{0.48\textwidth}
\includegraphics[width=\linewidth]{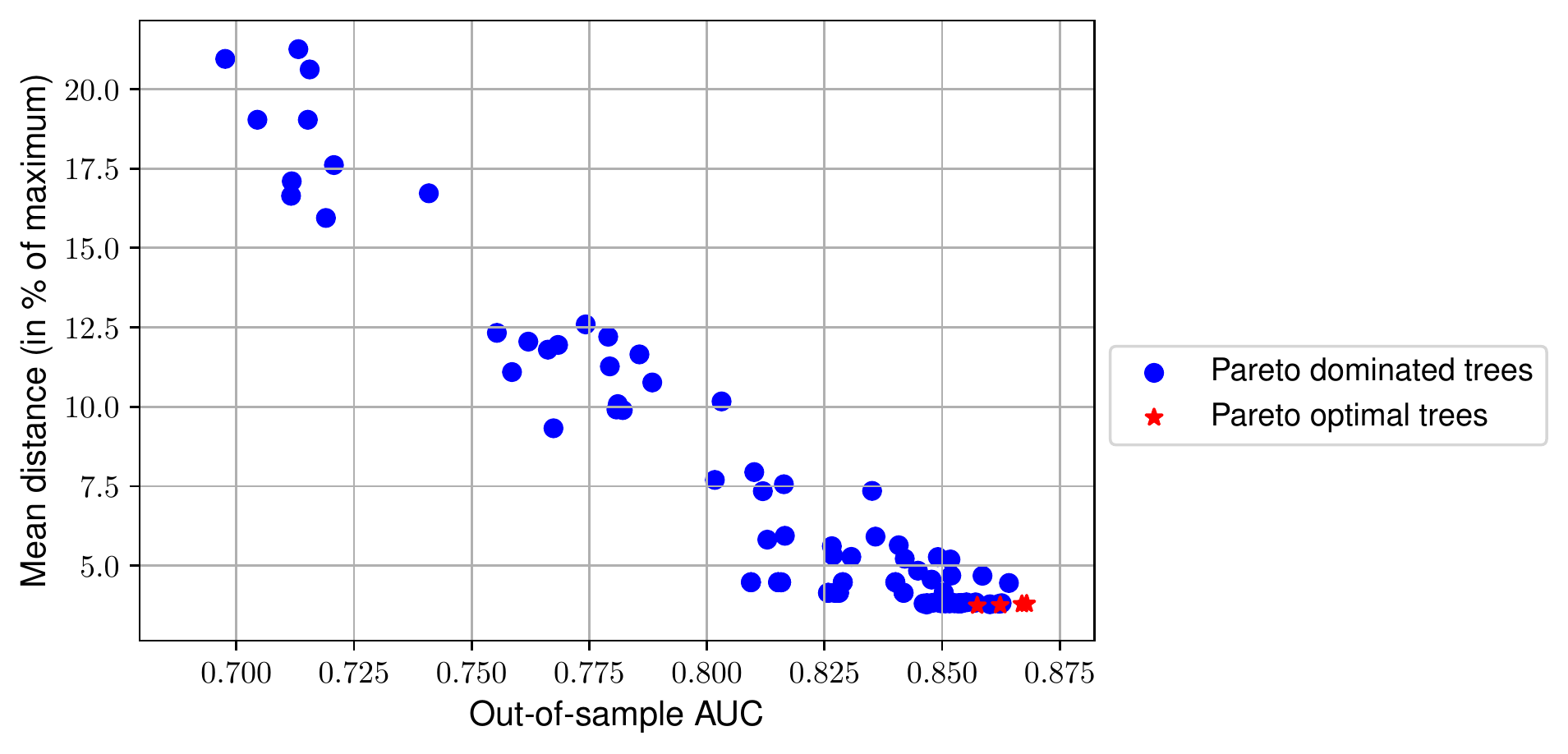}
\caption{Splenic injury} \label{fig:frontier-spleen}
\end{subfigure}\hspace*{\fill}
\begin{subfigure}{0.48\textwidth}
\includegraphics[width=\linewidth]{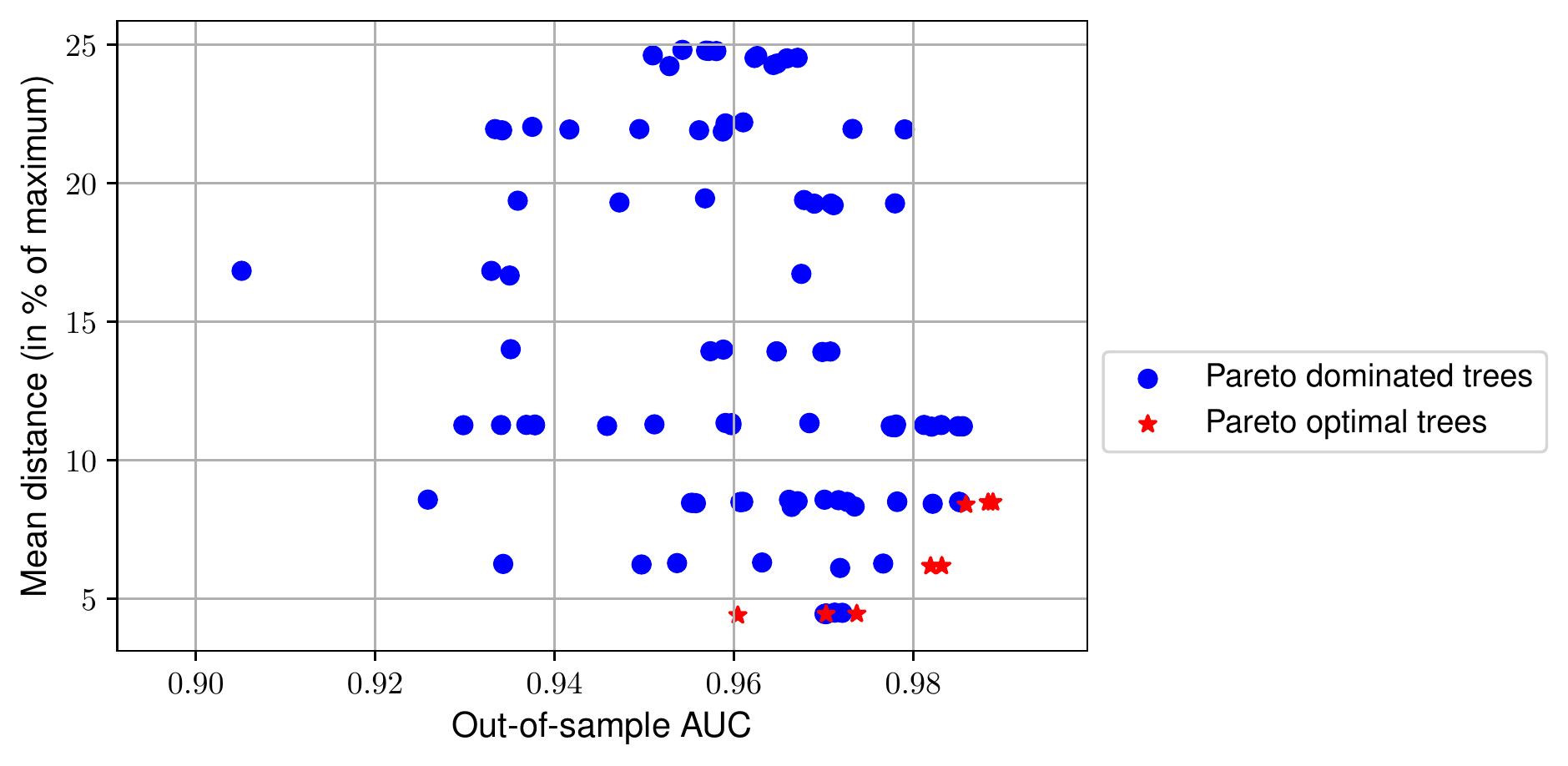}
\caption{Breast cancer} \label{fig:frontier-bc}
\end{subfigure}

\caption{Trade-off between stability and predictive performance for collection of trees.} \label{fig:frontiers}
\end{figure}

\begin{figure}[!ht]
\centering
\includegraphics[width=0.6\linewidth]{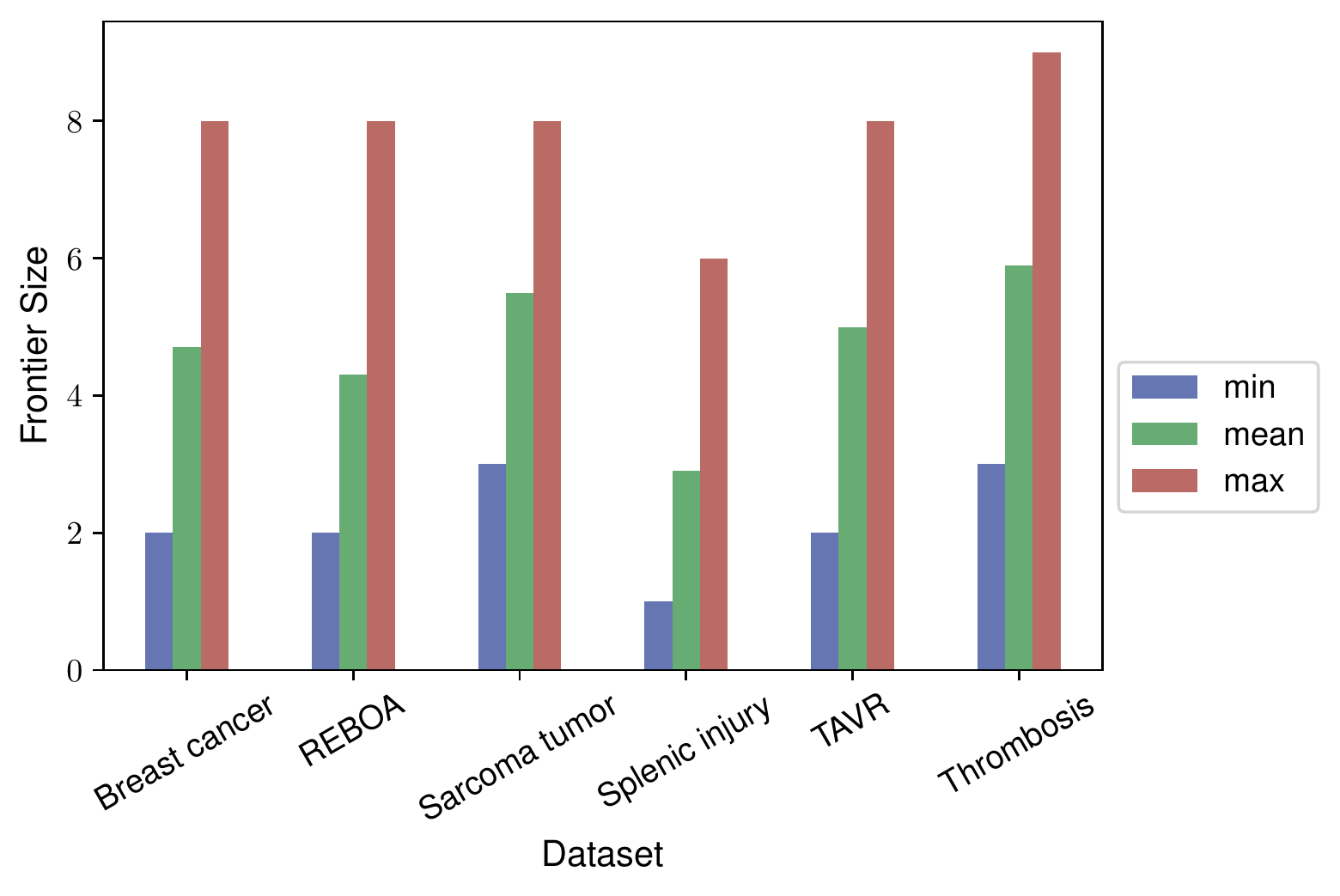} 
\caption{Minimum, mean, and maximum number of Pareto optimal trees for each dataset across 10 data splits.} \label{fig:bar-frontier}
\end{figure}

\subsection{The Effect of Stability on Predictive Performance} \label{ss:stability-accuracy}

This section aims to further understand the trade-off between stability and predictive performance. In Figure \ref{fig:bar-accuracy}(left), we benchmark, for each case study and in terms of their out-of-sample AUC, two Pareto optimal trees (the AUC-maximizing, referred to as ``\verb|CART Pareto AUC|'', and the distance-minimizing, referred to as ``\verb|CART Pareto Distance|'') against the best tree obtained using a standard 5-fold cross-validation procedure (referred to as ``\verb|CART CV|'') and against random forest (``\verb|RF'|'). Random forest is known to significantly improve upon the stability and predictive performance of decision trees \cite{breiman2001random}, at the expense of interpretability. We make the following observations:
\begin{itemize}
    \item In all case studies, \verb|CART Pareto AUC| significantly outperforms \verb|CART Pareto Distance| and \verb|CART CV|. In the Sarcoma tumor and TAVR case studies, \verb|CART Pareto AUC| outperforms \verb|RF|, whereas in Breast cancer it competes closely.
    \item \verb|CART Pareto Distance| outperforms \verb|CART CV| in three case studies in terms of mean AUC and in four case studies in terms of standard deviation. This leads to the conclusion that \verb|CART Pareto Distance| strictly dominates \verb|CART Pareto CV|; thus, in the sequel, we compare the two Pareto optimal trees against \verb|RF|.
\end{itemize}

In addition, for each case study, we report aggregated results on the mean distance of \verb|CART Pareto AUC| and \verb|CART Pareto Distance| from the trees in the first batch (Figure \ref{fig:bar-accuracy}(right)). We see that, except for the Breast cancer case study, the mean distance, for both trees, never exceeds 5\% of the maximum possible distance, while the standard deviation of the distance (obtained over multiple repetitions of each experiment) is also small. This suggests that the proposed methodology enforces a significant level of stability.

\begin{figure}[!ht]
\begin{subfigure}{0.48\textwidth}
\includegraphics[width=\linewidth]{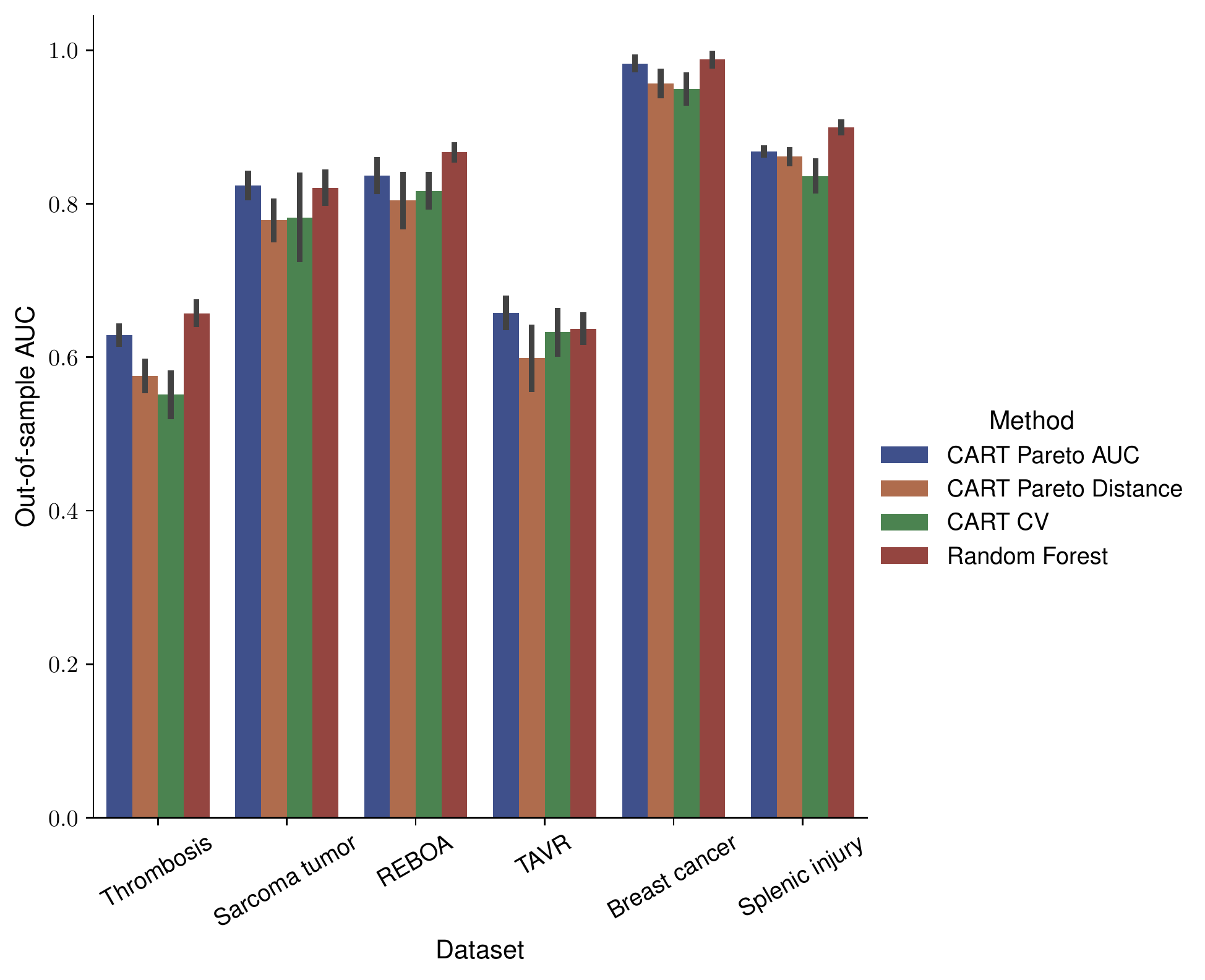} 
\end{subfigure}\hspace*{\fill}
\begin{subfigure}{0.48\textwidth}
\includegraphics[width=\linewidth]{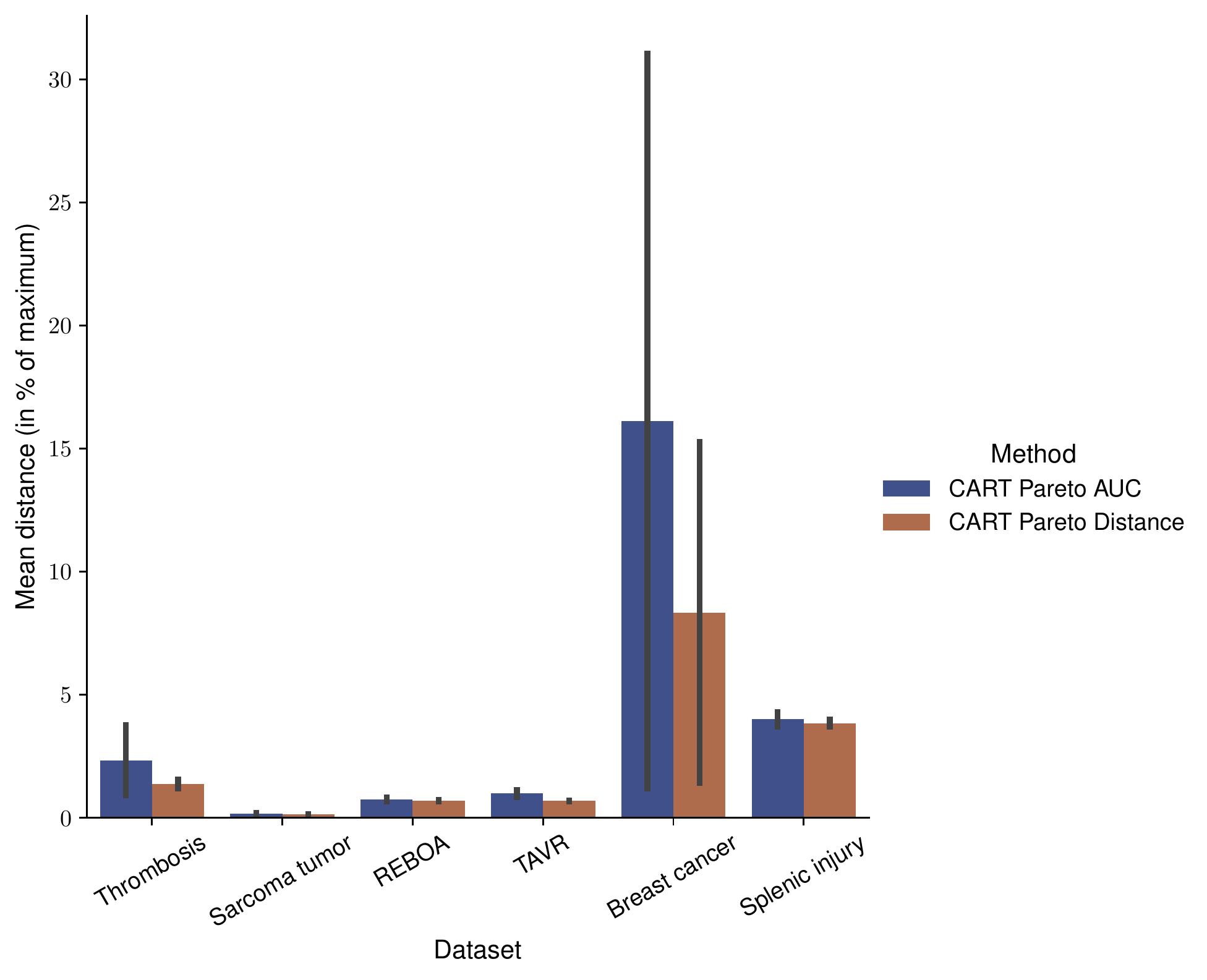} 
\end{subfigure}
\caption{Mean and standard deviation of out-of-sample AUC (left) and stability (right) for each dataset across 10 data splits. In the left figure, we compare the AUC-maximizing Pareto optimal tree, the distance-minimizing Pareto optimal tree, the tree obtained using cross-validation, and random forest. In the right figure, we study stability in terms of the mean distance between the AUC-maximizing/distance-minimizing Pareto optimal tree and the trees obtained using the first batch of training data.} \label{fig:bar-accuracy}
\end{figure}

\subsection{The Effect of Stability on Interpretability} \label{ss:stability-interpretability}

We now investigate what implications our notion of stability has on the interpretability of the selected trees. To quantify interpretability, we compare the tree depth (Figure \ref{fig:bar-interpretability}(left)) and number of nodes in the tree (Figure \ref{fig:bar-interpretability}(right)) between \verb|CART Pareto AUC|, \verb|CART Pareto Distance|, and the largest tree in \verb|RF|. In five out of six case studies, the proposed methodology results in much simpler trees compared to \verb|RF|, which, on average, are half as deep and consist of two to ten times fewer nodes. Moreover, in four out of six case studies, the \verb|CART Pareto AUC| tree is deeper and consists of more nodes compared to the \verb|CART Pareto Distance| tree. This suggests that the proposed notion of stability is more likely to result in simpler and hence more interpretable trees.

\begin{figure}[!ht]
\begin{subfigure}{0.48\textwidth}
\includegraphics[width=\linewidth]{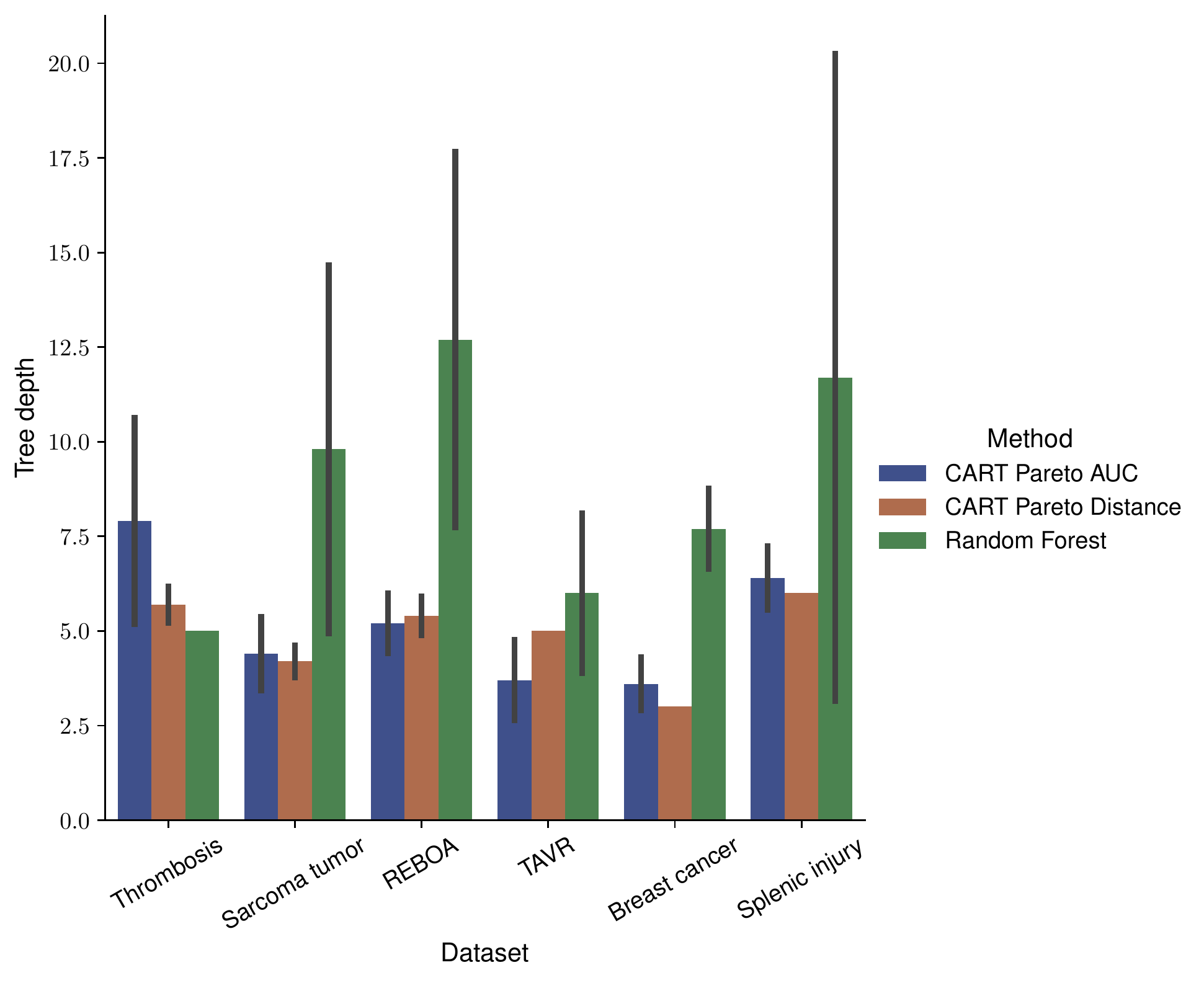} 
\end{subfigure}\hspace*{\fill}
\begin{subfigure}{0.48\textwidth}
\includegraphics[width=\linewidth]{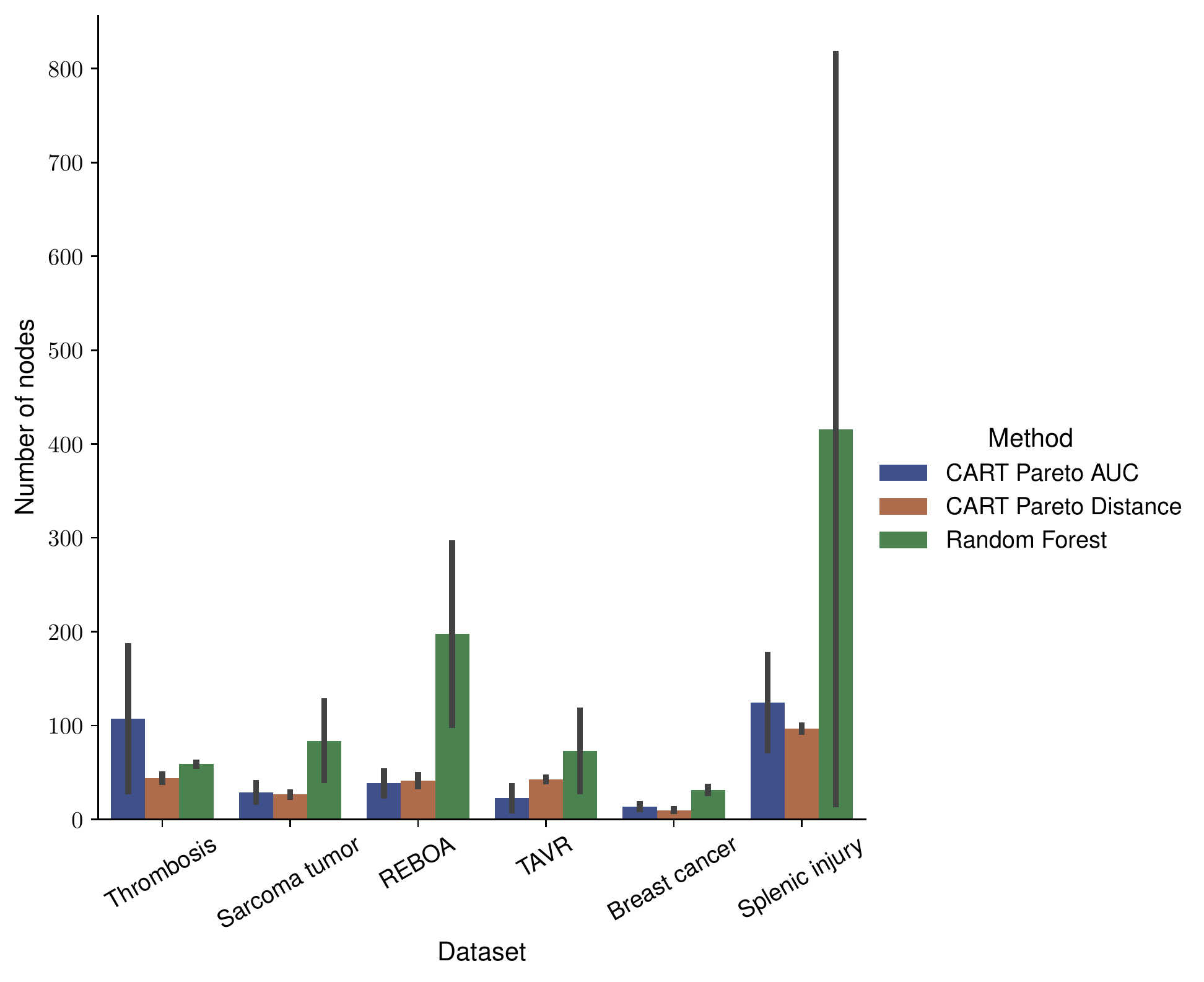} 
\end{subfigure}
\caption{Mean and standard deviation of tree depth (left) and number of nodes (right) for each dataset across 10 data splits. We compare the AUC-maximizing Pareto optimal tree, the distance-minimizing Pareto optimal tree, and the largest tree in random forest.} \label{fig:bar-interpretability}
\end{figure}

\subsection{The Effect of Stability on Feature Importance} \label{ss:stability-feature-selection}

We now study the effect of the proposed notion of stability on feature importances (or, equivalently, relevances), measured through the Gini importance \citep{hastie2009elements, menze2009comparison}. The Gini importance is a commonly used feature importance proxy that relies on the Gini impurity, which, in turn, measures the homogeneity of the target variable within different nodes in the tree. More concretely, the Gini impurity is a measure of how often a randomly chosen training data point from a node would be incorrectly labeled if it was randomly labeled according to the distribution of labels in the node. The Gini importance of a feature is computed as the (normalized) total reduction of the Gini impurity brought by that feature, by measuring how often the feature was selected for a split and how large its overall discriminative value was. 

We compare the standard deviation of feature importances averaged across all features (Figure \ref{fig:bar-feature-selection}(left)) and the total number of distinct features ranked as top-3 based on their feature importances (Figure \ref{fig:bar-feature-selection}(right)) between \verb|CART Pareto AUC|, \verb|CART Pareto Distance|, and \verb|RF|. \verb|RF| leads to smaller standard deviation and fewer distinct features marked as important; between \verb|CART Pareto AUC| and \verb|CART Pareto Distance|, the latter achieves smaller standard deviation in four out of six studies and marks more features as important in only one out of six studies hence suggesting that the proposed stability notion is positively correlated with having small deviations in feature importances.

\begin{figure}[!ht]
\begin{subfigure}{0.48\textwidth}
\includegraphics[width=\linewidth]{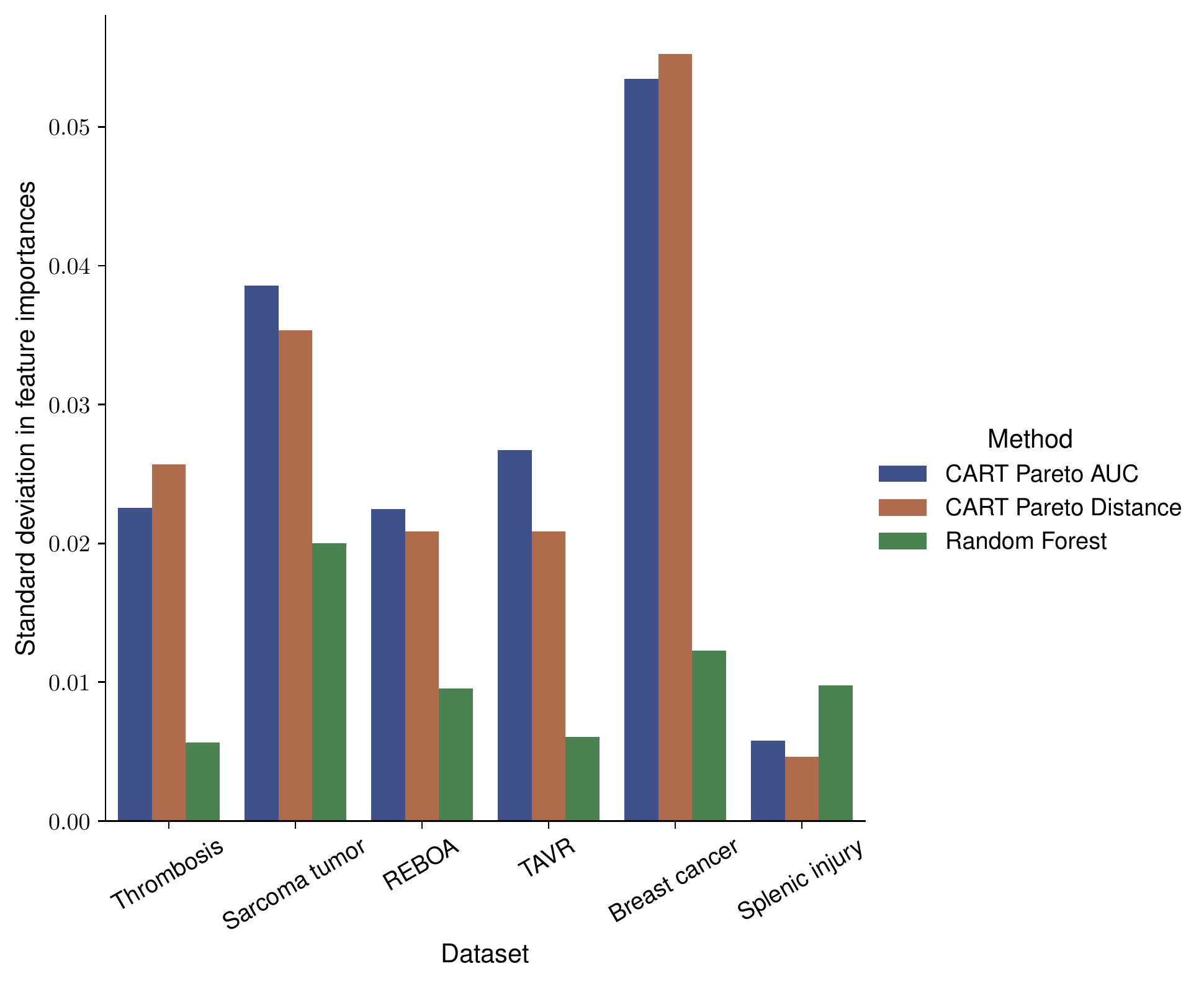} 
\end{subfigure}\hspace*{\fill}
\begin{subfigure}{0.48\textwidth}
\includegraphics[width=\linewidth]{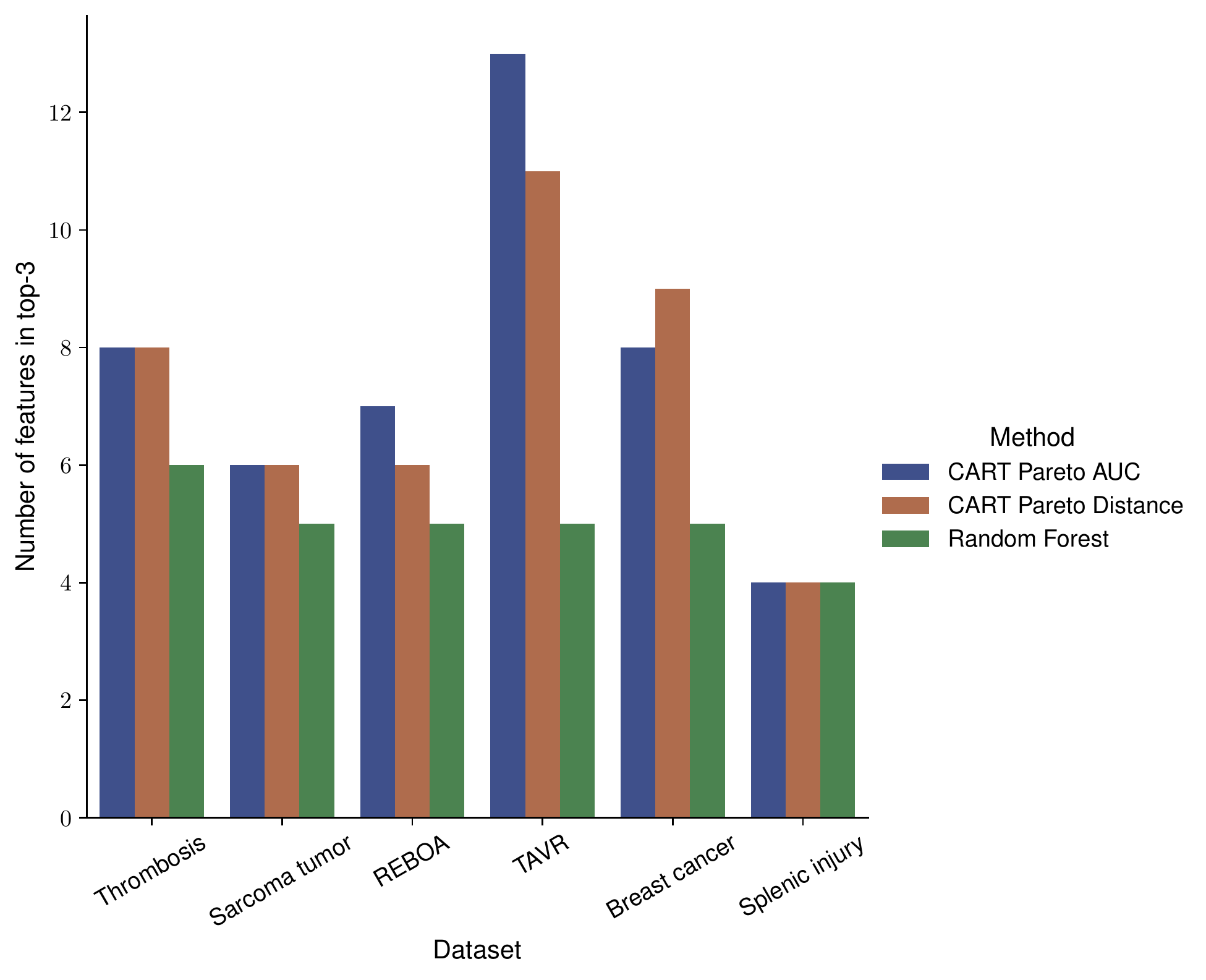}
\end{subfigure}
\caption{Standard deviation of feature importances, averaged across all features (left), and total number of distinct features ranked as top-3 based on their feature importances (right) for each dataset across 10 data splits. } \label{fig:bar-feature-selection}
\end{figure}

\subsection{Main Takeaways} \label{ss:pareto-takeaways}

We close this section by providing, in  Table \ref{tab:takeaways}, a summary comparison between the AUC-maximizing and the distance-minimizing Pareto optimal trees across all experiments. The price of a 38\% improvement in stability is, on average, 0.04 (or 4.625\%) in AUC. The distance-minimizing Pareto optimal tree reduces the standard deviation in feature importances by 3.6\%, the number of distinct important features by 4.4\%, the number of nodes in the tree by 22.2\%, and the tree depth by 6.2\%.

\begin{table}[!ht]
\caption{Aggregated comparison between the two extreme Pareto optimal trees.} \label{tab:takeaways}
\resizebox{\textwidth}{!}{
\begin{tabular}{cccccccccc}
\toprule
\textbf{Method} &
  \textbf{AUC} &
  \textbf{Distance} &
  \textbf{Feat. Import. Std.} &
  \textbf{Feat. in Top-3} &
  \textbf{Nodes} &
  \textbf{Tree Depth} \\ \midrule
\texttt{CART Pareto AUC} &
  0.8 (0.013) &
  4.056 (2.818) &
  0.028 &
  7.667 &
  55.867 &
  5.2 \\
\texttt{CART Pareto Distance} &
  0.763 (0.024) &
  2.511 (1.219) &
  0.027 &
  7.333 &
  43.5 &
  4.883 \\ \bottomrule
\end{tabular}}
\end{table}

\section{Case Studies} \label{s:case-studies}

In this section, we describe in detail the six real-world case studies we have used throughout the paper (Section \ref{ss:description-case-studies}) and analyze the trees trained using the proposed framework (Section \ref{ss:stability-qualitative}). 

\subsection{Description of the Case Studies} \label{ss:description-case-studies}

All case studies come from the health care space due to the relevance of interpretable ML in this space. In the first five case studies, we obtain the data from collaborations with major US hospitals, including the Massachusetts General Hospital (MGH), the Hartford Hospital (HH), and the Memorial Sloan Kettering Cancer Center (MSK); for reproducibility purposes, in the last case study (``breast cancer''), we use a publicly available dataset. Table \ref{tab:datasets} provides a description of the datasets we use, including the origin of each dataset, the number of samples and features, and the outcome prevalence (i.e., the proportion of samples with the health outcome under consideration).

\begin{table}[!ht]
\caption{Summary of datasets we use in our case studies.} \label{tab:datasets}
\resizebox{\textwidth}{!}{
\begin{tabular}{ccccc}
\toprule
\textbf{Dataset} & \textbf{Origin} & \textbf{Number of Samples} & \textbf{Number or Features} & \textbf{Outcome Prevalence (in \%)} \\ \midrule
Thrombosis     & MGH Trauma Department        & 21,549 & 35 & 1.59  \\
Sarcoma tumor  & MSK                          & 930    & 17 & 16.56 \\
REBOA          & MGH Trauma Department        & 10,000 & 30 & 31.27 \\
TAVR           & HH Cardiovascular Department & 2,148  & 42 & 15.18 \\
Splenic injury & MGH Trauma Department        & 35,954 & 41 & 6.1   \\
Breast cancer  & UCI                          & 569    & 30 & 37.26 \\ \bottomrule
\end{tabular}
}
\end{table}

\paragraph{Thrombosis.} We are interested in predicting the risk of deep vein thrombosis (DVT) after endovenous thermal ablation. The features include demographic factors (e.g., age, sex, ethnicity) as well as categorical (e.g., cigarette smoking history, wound infection, baseline dyspnea) and continuous (e.g., preoperative measurements such as creatinine, hematocrit, platelet) risk factors. 

DVT is the most common cause of chronic venous disease, a widespread disease with an annual incidence of 1-2\% and prevalence up to 73\% in women and 56\% in men. Endovenous treatments, such as thermal and laser ablation, are safe and effective, offering better outcomes and lower complications compared to traditional surgery. Yet, there are still risks to these procedures, including the development of a DVT, with a complication rate of around 1\%, due to the subsequent risk of a pulmonary embolism. The risk of developing DVT ranges substantially in each individual patient due to specific risk factors, which motivates the study and development of prediction tools that estimate the risk of DVT after endovenous ablation \citep{marsh2010deep}.


\paragraph{Sarcoma tumor.} We examine the effect of radiotherapy on reducing local recurrence within five years to patients with sarcoma tumor. The features include demographic factors (e.g., age at surgery), categorical (e.g., radiosensitivity) and continuous (e.g., primary tumor size) risk factors, and treatment-related factors.

Sarcomas are rare cancers that develop in the bones and soft tissues, including fat, muscles, blood vessels, nerves, deep skin tissues, and fibrous tissues. According to the National Cancer Institute, about 12,000 cases of soft tissue sarcomas and 3,000 cases of bone sarcomas are diagnosed in the U.S. each year. Sarcoma is treated with a combination of chemotherapy, radiation therapy, and surgery. Patients who have received radiation therapy for previous cancers may have a higher risk of developing a sarcoma. E.g., after treatment of primary soft tissue sarcomas, 11\% to 14\% of patients develop local recurrence \citep{eilber2003high}, which may require additional surgery, radiotherapy, or even amputation, and may predict decreased overall survival. Most local recurrences arise in the first 5 years after diagnosis \citep{gadd1993development}. 


\paragraph{REBOA.} We study whether, using ML, we can decrease the misuse of resuscitative endovascular balloon occlusion of the aorta (REBOA) in hemodynamically unstable blunt trauma patients. The features include demographic factors, categorical (e.g., Glasgow Coma Scale) and continuous (e.g., pulse and respiratory rates) risk factors, and treatment-related factors.

REBOA is a procedure that involves placement of an endovascular balloon in the aorta to control bleeding, augment afterload, and maintain blood pressure temporarily in traumatic hemorrhagic shock. The balloon blocks the artery and temporarily stops the blood flow giving doctors time to operate, but maintains blood circulation in the brain and heart. However, the parts of the body below the balloon are cut off from the normal blood flow, which may result in short- or longer-term problems \citep{okada2017resuscitative,jansen2022effectiveness}.


\paragraph{TAVR.} We investigate whether using the appropriate valve type in a transcatheter aortic valve replacement (TAVR) procedure can reduce the need for pacemaker. The features include demographic factors, categorical (e.g., diabetes) and continuous (e.g., left ventricular ejection fraction) risk factors, and treatment-related factors (manufacturer and type of the valve). TAVR is a minimally invasive heart procedure to replace a thickened aortic valve that cannot fully open, in which case blood flow from the heart to the body is reduced. TAVR can help restore blood flow and may be an option for people who are at risk of complications from surgical aortic valve replacement: in those patients, TAVR significantly reduces the rates of death and cardiac symptoms \citep{carabello2011transcatheter}. Nevertheless, there are multiple risks associated with TAVR, e.g., problems with the replacement valve (e.g., the valve slipping out of place or leaking), arrhythmias, and the need for a pacemaker. 


\paragraph{Splenic injury.} We explore how different treatments affect mortality of victims of blunt splenic injury. The features include demographic factors, categorical (e.g., splenic injury grade) and continuous (e.g., pulse and respiratory rate) risk factors, and treatment-related factors (splenectomy, angioembolization, or observation).

Blunt splenic injury occurs when a significant impact from some outside source (e.g., automobile accident) damages or ruptures the spleen. The traditional treatment has been splenectomy, the surgical procedure that partially or completely removes the spleen. The spleen is an important organ in regard to immunological function due to its ability to efficiently destroy encapsulated bacteria; its removal runs the risk of overwhelming post-splenectomy infection, a medical emergency and rapidly fatal disease caused by the inability of the body's immune system to properly fight infection following splenectomy \citep{taniguchi2014overwhelming}. Therefore, whenever possible, splenectomy is avoided to prevent the resulting permanent susceptibility to bacterial infections: most small, and some moderate-sized lacerations in stable patients are managed with hospital observation and sometimes transfusion rather than surgery; angioembolization, blocking off of the hemorrhaging vessels, is a less invasive treatment \citep{thompson2006novel}.


\paragraph{Breast cancer.} We predict whether a breast cancer is benign or malignant using features computed from a digitized image of a fine needle aspirate of a breast mass. The features describe characteristics of the cell nuclei present in the image and include, specifically, for each nucleus, information about its radius, texture, perimeter, area, smoothness, compactness, concavity, number of concave points, summetry, and fractal dimension \citep{wolberg1994machine,mangasarian1995breast}. The data is publicly available at the UCI ML repository at \url{https://archive.ics.uci.edu/ml/datasets/breast+cancer+wisconsin+(diagnostic)}. 

\subsection{Qualitative Analysis of the Selected Trees}  \label{ss:stability-qualitative}

We now take a qualitative approach to analyzing the stability of the trees trained using the proposed framework. For each case study, we choose the tree from $\mathcal{T}^\star$ that maximizes Equation~\eqref{eq:select-tree}. We show, for each such tree, the first two levels of splits (at the root node, its left, and its right child) in Figure \ref{fig:trees}. More specifically, for each split node, we present the split feature and threshold, the proportion of total samples, the proportion of samples from each class, and the class label at that node. We compare the selected trees with the results given in Table \ref{tab:tree-analysis}, where, for each case study and among all trees in the full collection $\mathcal{T}$, we record the two most commonly selected features in each of the first three splits, along with their selection frequencies.

\begin{figure}[!ht]
\centering
\begin{subfigure}{0.48\textwidth}
\includegraphics[width=\linewidth]{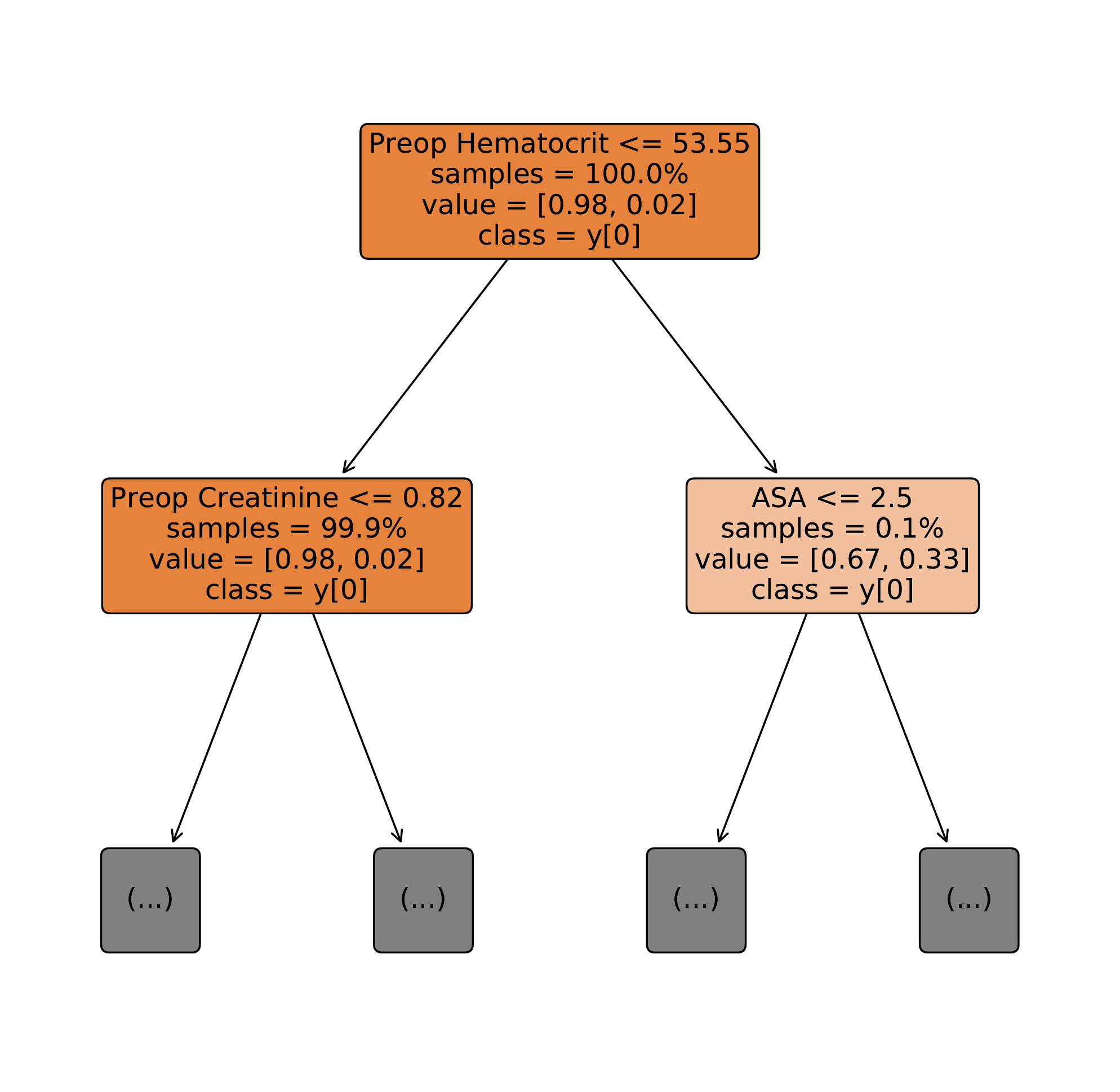}
\caption{Thrombosis}  \label{tab:tree-thrombosis}
\end{subfigure}\hspace*{\fill}
\begin{subfigure}{0.48\textwidth}
\includegraphics[width=\linewidth]{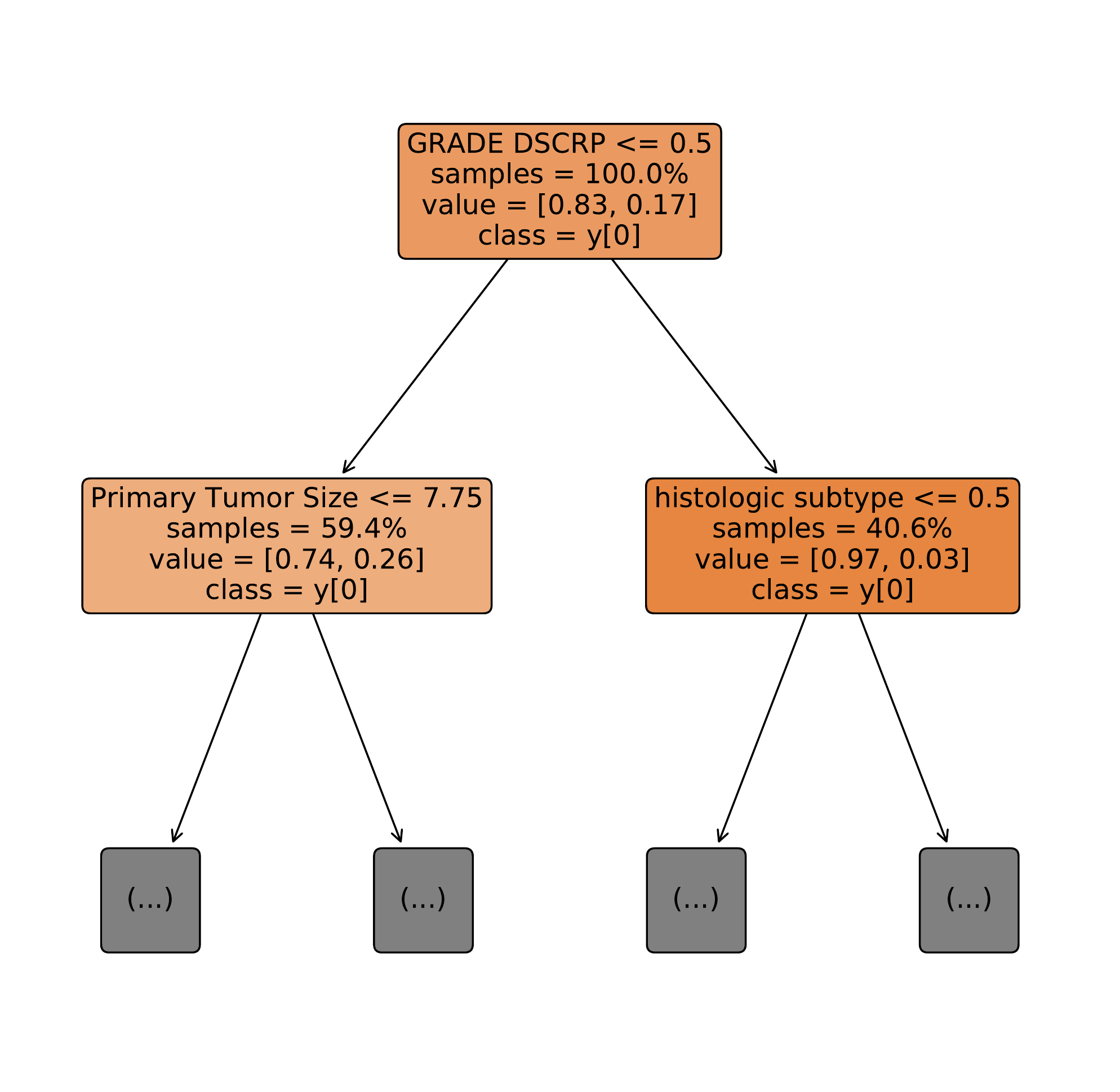}
\caption{Sarcoma Tumor} \label{tab:tree-sarcoma}
\end{subfigure}

\end{figure}%
\begin{figure}[ht]\ContinuedFloat

\medskip
\begin{subfigure}{0.48\textwidth}
\includegraphics[width=\linewidth]{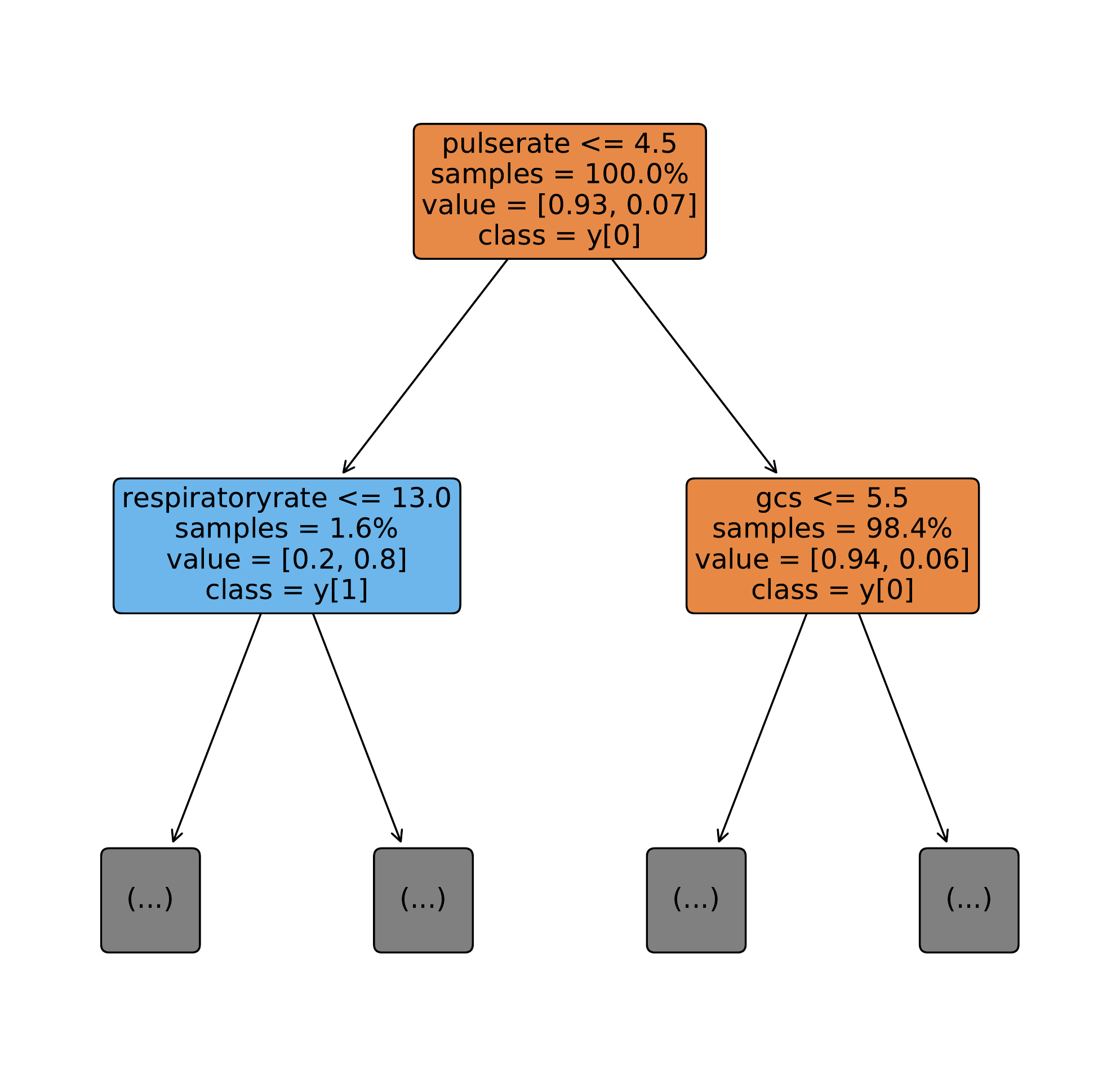}
\caption{REBOA} \label{tab:tree-reboa}
\end{subfigure}\hspace*{\fill}
\begin{subfigure}{0.48\textwidth}
\includegraphics[width=\linewidth]{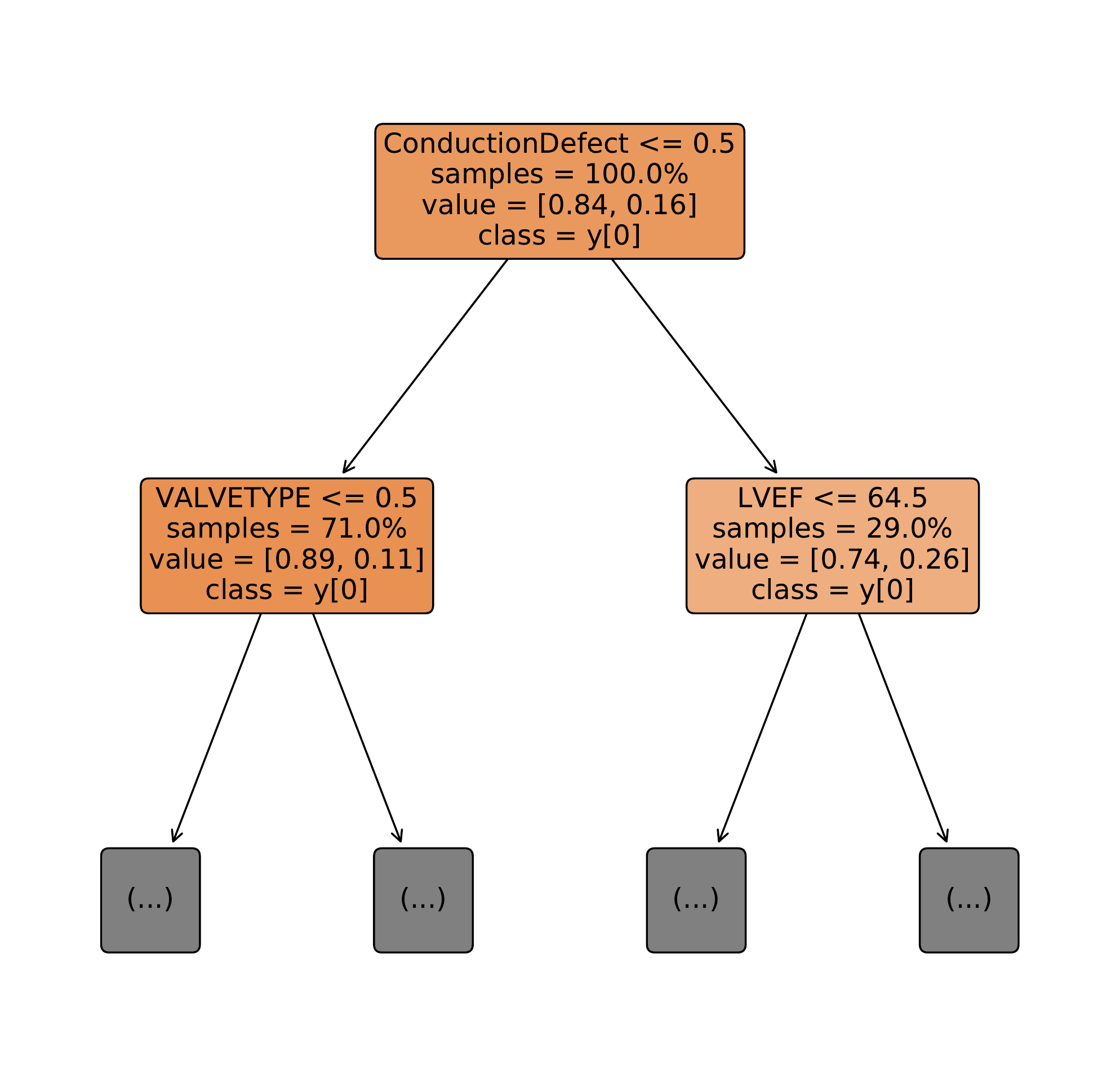}
\caption{TAVR} \label{tab:tree-tavr}
\end{subfigure}

\end{figure}%
\begin{figure}[ht]\ContinuedFloat

\medskip
\begin{subfigure}{0.48\textwidth}
\includegraphics[width=\linewidth]{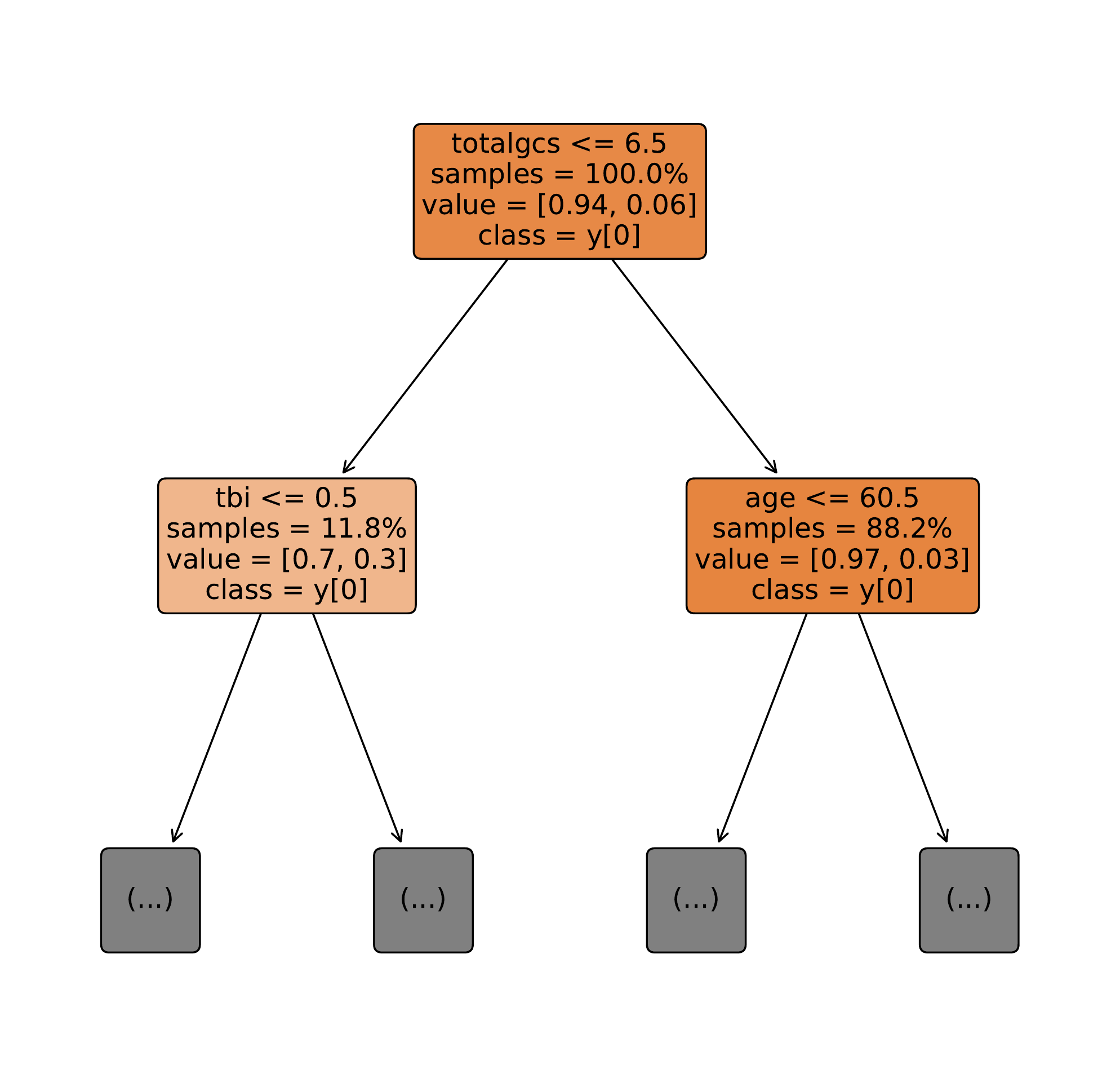}
\caption{Splenic injury} \label{tab:tree-spleen}
\end{subfigure}\hspace*{\fill}
\begin{subfigure}{0.48\textwidth}
\includegraphics[width=\linewidth]{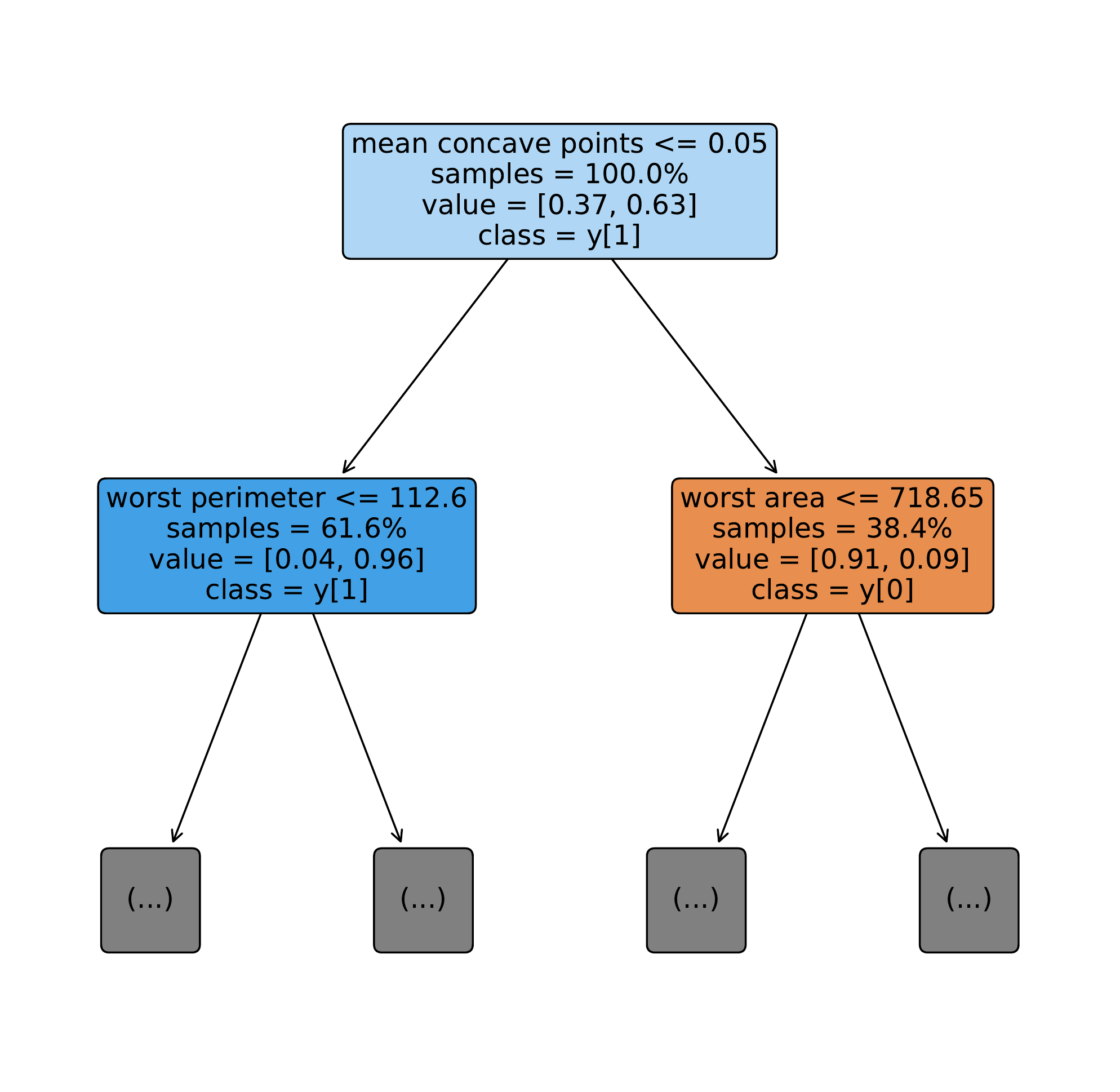}
\caption{Breast cancer} \label{tab:tree-breast}
\end{subfigure}

\caption{Visualization of the first two levels of splits in the Pareto optimal tree obtained using Equation~\eqref{eq:select-tree} for each case study.} \label{fig:trees}
\end{figure}

\begin{table}[!ht]
\caption{Two most commonly selected features (and their corresponding selection frequencies) in each of the first two levels of splits, among all trained trees and across all repetitions for each case study.} \label{tab:tree-analysis}
\resizebox{\textwidth}{!}{
\begin{tabular}{cccccc}
\textbf{Dataset} &
  \textbf{Split} &
  \textbf{Feature 1} &
  \textbf{Frequency (in \%)} &
  \textbf{Feature 2} &
  \textbf{Frequency (in \%)} \\ \toprule
& Root  & Preop Hematocrit     & 59.76  & Preop Platelet      & 14.33 \\
& Left  & Preop Platelet       & 18.86  & Age                 & 16.48 \\
\multirow{-3}{*}{Thrombosis}     & Right & \textit{Leaf node}                 & 60.48  & Preop Hematocrit    & 6.38  \\ \midrule
& Root  & GRADE DSCRP          & 47.33  & Chemo 2             & 25.00 \\
& Left  & Primary Tumor Size   & 59.00  & Chemo 2             & 12.33 \\
\multirow{-3}{*}{Sarcoma tumor} &
  Right &
  GRADE DSCRP &
  30.33 &
  margin width &
  29.10 \\ \midrule
& Root  & pulserate            & 81.00  & gcs                 & 10.00 \\
& Left  & age                  & 24.10  & respiratoryrate     & 18.71 \\
\multirow{-3}{*}{REBOA}          & Right & gcs                  & 90.00  & respiratoryrate     & 5.33  \\ \midrule
& Root  & ConductionDefect     & 63.00  & VALVETYPE           & 14.00 \\
& Left  & VALVETYPE            & 44.00  & ConductionDefect    & 13.00 \\
\multirow{-3}{*}{TAVR}           & Right & LVEF                 & 19.67  & Area-Oversize       & 16.00   \\ \midrule
& Root  & totalgcs             & 100.00 & -                & -  \\
& Left  & tbi                  & 70.00  & age                 & 30.00 \\
\multirow{-3}{*}{Splenic injury} & Right & age                  & 100.00 & -                & -  \\ \midrule
& Root  & worst perimeter      & 38.52  & mean concave points & 22.00 \\
& Left  & worst concave points & 35.38  & worst area          & 12.00 \\
\multirow{-3}{*}{Breast cancer}  & Right & worst texture        & 17.38  & \textit{Leaf node}               & 13.00 \\ \bottomrule
\end{tabular}
}
\end{table}

In summary, the selected trees are representative of the aggregate statistics. In \textit{TAVR} (Figure \ref{tab:tree-tavr}) and \textit{Splenic injury} (Figure \ref{tab:tree-spleen}), the trees perform all splits in the first two levels on the most commonly selected features across all trained trees, suggesting that the selected trees are indeed stable. For \textit{TAVR}, these features are the existence of a conduction defect (the root split is performed on this feature in 63\% of the trees), the type of valve used (left split in 44\%), and the left ventricular ejection fraction (right split in 20\%). For \textit{Splenic injury}, these features are the Glasgow Coma Scale (abbreviated ``totalgcs'', root split in 100\%), the existence of traumatic brain injury (abbreviated ``tbi'', left split in 70\%), and the age (right split in 100\%). \textit{REBOA} (Figure \ref{tab:tree-reboa}), the tree splits on the pulse rate (root split in 81\%) and then on the respiratory rate (second most common left split feature in 19\%) and the Glasgow Coma Scale (abbreviated ``gcs'', right split in 90\%). In \textit{Sarcoma tumor} (Figure \ref{tab:tree-sarcoma}), the tree splits on the sarcoma grade (suggesting how abnormal the cancer cells look under a microscope --- root split in 47\%) and then on the primary tumor size (left split in 59\%) and the histologic subtype (which is not among the most commonly selected features). In \textit{Thrombosis} (Figure \ref{tab:tree-thrombosis}), the tree splits on the preoperative hematocrit levels (root split in 60\%) and then on the preoperative creatinine levels and the ASA score, a metric to determine if someone is healthy enough to tolerate surgery and anesthesia (which are not among the most commonly selected features). Finally, in \textit{Breast cancer} (Figure \ref{tab:tree-breast}), the tree splits on the mean concave points feature (second most common root split feature in 22\%) and then on the worst perimeter (most common root split feature in 39\%) and the worst area (second most common left split feature in 12\%).




\section{Conclusion} 

In this paper, we have developed a methodology to improve the stability of decision tree models. The proposed methodology enables us to investigate trade-offs that are inherent to decision tree models and train stable, ``Pareto optimal'' decision trees; we demonstrate the value of the proposed approach through six extensive quantitative and qualitative case studies from the health care space, where interpretability is essential. Further, we introduce a novel distance metric for decision trees, which has been missing from the ML literature, and use it to determine a tree's level of stability; the proposed distance metric may be of independent interest and can be used in different contexts, including, e.g., as a new way to impose regularization in decision tree models. 

We conclude by returning to Leo Breiman's question, ``is there a more stable single-tree version of CART?'' Our work suggests that such a tree may indeed exist and, although it is unlikely to achieve the levels of stability of procedures that rely on averaging, the proposed Pareto optimal trees achieve improved predictive power (see, e.g., Section \ref{ss:stability-accuracy}), interpretability (see Section \ref{ss:stability-interpretability}), and are structurally stable (see Section \ref{ss:stability-qualitative}).

\section*{Acknowledgements}
We are grateful to Yu Ma for working with us on the data and for fruitful discussions, and to Agni Orfanoudaki and Wolfram Wiesemann for valuable comments. 


\bibliographystyle{informs2014}
\bibliography{references.bib}


\begin{APPENDICES}

\section{Technical Proofs} \label{appx:proofs}

\subsection{Proof of Proposition~\ref{lem:metric}} \label{appx:proofs-metric}
Recall that we represent a path $p$ as a collection of two vectors (upper and lower bounds for numerical features), a matrix (categories for categorical features), and a scalar (class label): $(\bm u^p, \bm l^p, \bm C^p, k^p) \in \mathbb{R}^{|\mathcal{N}|} \times \mathbb{R}^{|\mathcal{N}|} \times \{0,1\}^{|\mathcal{C}| \times \max_j c_j} \times [K] \ := \ \mathcal{P}$. Without loss of generality, we flatten matrix $\boldsymbol C^p$ into a vector of appropriate dimension. We first prove the following auxiliary lemma:
\begin{lemma} \label{lem:metric-path}
    Let $\mathcal{P}$ denote the set of all paths of depth $D$ and $p, q \in \mathcal{P}$. Then
    $$d(p,q) = \sum_{j \in \mathcal{N}} \frac{\left|u^p_j - u^q_j\right| + \left|l^p_j - l^q_j\right|}{2(u_j - l_j)} + \sum_{j \in \mathcal{C}} \frac{\left\|\bm c^p_j-\bm c^q_j\right\|_1}{c_j} + \lambda \cdot \mathbbm{1}_{(k^p \neq k^q)}$$
    is a metric mapping $\mathcal{P}\times \mathcal{P} \mapsto \mathbb{R}.$
\end{lemma}
\begin{proof}{Proof of Lemma~\ref{lem:metric-path}.}
    We need to show that \textbf{(i)} $d(p,p)=0$, \textbf{(ii)} if $p \neq q$, then $d(p,q)>0$, \textbf{(iii)} $d(p,q)=d(q,p)$, \textbf{(iv)} $d(p,q) \leq d(p,r) + d(r,q)$. Axioms \textbf{(i)}-\textbf{(iii)} are trivially satisfied. We now prove axiom \textbf{(iv)}. Denoting $[\boldsymbol \mu]_j = \frac{1}{2(u_j-l_j)}>0$, $[\boldsymbol \nu]_j = \frac{1}{c_j}>0$, and by $\boldsymbol x \odot \boldsymbol y$ the element-wise product between vectors $\boldsymbol x$ and $\boldsymbol  y$, we have:
    \begin{equation*}
        \begin{split}
            d(p,q) & \ = \ \| \boldsymbol \mu \odot (\boldsymbol u^p - \boldsymbol u^q) \|_1 + \| \boldsymbol \mu \odot (\boldsymbol l^p - \boldsymbol l^q) \|_1 + \| \boldsymbol \nu \odot (\boldsymbol C^p - \boldsymbol C^q) \|_1 + \lambda \cdot \mathbbm{1}_{(k^p \neq k^q)} \\
            d(p,r) & \ = \ \| \boldsymbol \mu \odot (\boldsymbol u^p - \boldsymbol u^r) \|_1 + \| \boldsymbol \mu \odot (\boldsymbol l^p - \boldsymbol l^r) \|_1 + \| \boldsymbol \nu \odot (\boldsymbol C^p - \boldsymbol C^r) \|_1 + \lambda \cdot \mathbbm{1}_{(k^p=k^r)} \\
            d(r,q) & \ = \ \| \boldsymbol \mu \odot (\boldsymbol u^r - \boldsymbol u^q) \|_1 + \| \boldsymbol \mu \odot (\boldsymbol l^r - \boldsymbol l^q) \|_1 + \| \boldsymbol \nu \odot (\boldsymbol C^r - \boldsymbol C^q) \|_1 + \lambda \cdot \mathbbm{1}_{(k^r=k^q)} \\
        \end{split}
    \end{equation*}
    We apply the triangle inequality to the first summand:
    \begin{equation*}
        \begin{split}
            \| \boldsymbol \mu \odot (\boldsymbol u^p - \boldsymbol u^q) \|_1 & \ = \  \sum_{j \in \mathcal{N}} | \mu_j (u_j^p - u_j^q)| \\
            & \ = \  \sum_{j \in \mathcal{N}}  \mu_j | u_j^p - u_j^q | \\
            & \ \leq \  \sum_{j \in \mathcal{N}}  \mu_j (| u_j^p - u_j^r| + | u_j^r - u_j^q |) \\
            & \ = \  \sum_{j \in \mathcal{N}} | \mu_j (u_j^p - u_j^r)| + | \mu_j (u_j^r - u_j^q)| \\
            & \ = \  \| \boldsymbol \mu \odot (\boldsymbol u^p - \boldsymbol u^r) \|_1 + \| \boldsymbol \mu \odot (\boldsymbol u^r - \boldsymbol u^q) \|_1.
        \end{split}
    \end{equation*}
    Working similarly, we can show that 
    \begin{equation*}
        \begin{split}
            & \| \boldsymbol \mu \odot (\boldsymbol l^p - \boldsymbol l^q) \|_1 \leq 
            \| \boldsymbol \mu \odot (\boldsymbol l^p - \boldsymbol l^r) \|_1 + 
            \| \boldsymbol \mu \odot (\boldsymbol l^r - \boldsymbol l^q) \|_1, \\
            & \| \boldsymbol \nu \odot (\boldsymbol C^p - \boldsymbol C^q) \|_1 \leq 
            \| \boldsymbol \nu \odot (\boldsymbol C^p - \boldsymbol C^r) \|_1 + 
            \| \boldsymbol \nu \odot (\boldsymbol C^r - \boldsymbol C^q) \|_1.
        \end{split}
    \end{equation*}
    Moreover, we can show by case analysis that
    
    \begin{equation*}
        \begin{split}
            \lambda \cdot \mathbbm{1}_{(k^p \neq k^q)} \leq \lambda \cdot \mathbbm{1}_{(k^p=k^r)} + \lambda \cdot \mathbbm{1}_{(k^r=k^q)}.
        \end{split}
    \end{equation*}
    By summing the above four inequalities, we get that $d(p,q)$ satisfies the triangle inequality (axiom \textbf{(iv)}) $d(p,q) \leq d(p,r) + d(r,q)$, and is therefore a metric. \hfill $\square$
\end{proof}

We now proceed with the proof of Proposition~\ref{lem:metric}:
\begin{proof}{Proof of Proposition~\ref{lem:metric}}
    Let $\mathcal{T}$ denote the set of all trees of maximum depth $D$ and $\mathbbm{T}_1, \mathbbm{T}_2, \mathbbm{T}_3 \in \mathcal{T}$ with $\mathbbm{T}_1 \neq \mathbbm{T}_2 \neq \mathbbm{T}_3$. We assume, without loss of generality, that $\mathbbm{T}_1, \mathbbm{T}_2, \mathbbm{T}_3$ have the same number of paths $P$; for the more general case, we work similarly to the proof of Corollary~\ref{cor:relaxation} (by appending ``dummy'' paths to the trees with the fewer paths). 
    
    We need to show that \textbf{(i)} $d(\mathbbm{T}_1,\mathbbm{T}_1)=0$, \textbf{(ii)} $d(\mathbbm{T}_1, \mathbbm{T}_2)>0$, \textbf{(iii)} $d(\mathbbm{T}_1, \mathbbm{T}_2)=d(\mathbbm{T}_2, \mathbbm{T}_1)$, \textbf{(iv)} $d(\mathbbm{T}_1, \mathbbm{T}_3) \leq d(\mathbbm{T}_1, \mathbbm{T}_2) + d(\mathbbm{T}_2, \mathbbm{T}_3)$. Axiom \textbf{(i)} is satisfied by matching each path in Problem~\eqref{eq:distance} with itself and then applying Lemma~\ref{lem:metric-path}. Axiom \textbf{(ii)} is satisfied by observing that $\mathbbm{T}_1 \neq \mathbbm{T}_2$ (implying they differ in at least one path) and then applying Lemma~\ref{lem:metric-path}. Axiom \textbf{(iii)} is satisfied by again invoking Lemma~\ref{lem:metric-path} and noticing that the solution to Problem~\eqref{eq:distance} does not depend on the order of the input trees.  

    We now prove axiom \textbf{(iv)}. Let us index the paths in $\mathbbm{T}_1, \mathbbm{T}_2, \mathbbm{T}_3$ according to the optimal matching between $\mathbbm{T}_1$ and $\mathbbm{T}_2$ so that $ p^1_1 \leftrightarrow p^2_1, \ \dots, \ p^1_P \leftrightarrow p^2_P$, and between $\mathbbm{T}_2$  and $\mathbbm{T}_3$ so that $p^2_1 \leftrightarrow p^3_1, \ \dots, \ p^2_P \leftrightarrow p^3_P$; the subscript corresponds to the path index and the superscript corresponds to the tree index. Let us also consider a permutation of the paths in $\mathbbm{T}_3$ (denoted by $q$) according to the optimal matching between $\mathbbm{T}_1$ and $\mathbbm{T}_3$ so that $ p^1_1 \leftrightarrow q^3_1, \ \dots, \ p^1_P \leftrightarrow q^3_P$. By application of Lemma~\ref{lem:metric-path} we get that, for any $i \in [P]$,
    \begin{equation} \label{eqn:matched-paths}
        d(p_i^1, p_i^2) + d(p_i^2, p_i^3) \geq d(p_i^1, p_i^3).
    \end{equation}
    Taking the sum across all paths, yields
    \begin{equation*}
        \begin{split}
            d(\mathbbm{T}_1, \mathbbm{T}_2) + d(\mathbbm{T}_2, \mathbbm{T}_3)
            \ \overset{\mathrm{(a)}}{=} \ \sum_{i \in [P]} d(p_i^1, p_i^2) + d(p_i^2, p_i^3) 
            \ & \overset{\mathrm{(b)}}{\geq} \ \sum_{i \in [P]} d(p_i^1, p_i^3) \\
            \ & \overset{\mathrm{(c)}}{\geq} \ \sum_{i \in [P]} d(p_i^1, q_i^3)
            \ \overset{\mathrm{(d)}}{=} \ d(\mathbbm{T}_1, \mathbbm{T}_3),
        \end{split}
    \end{equation*}
    where \textbf{(a)} follows from the fact that $p^1_i \leftrightarrow p^2_i$ is the optimal matching between $\mathbbm{T}_1$ and $\mathbbm{T}_2$, and $p^2_i \leftrightarrow p^3_i$ is the optimal matching between $\mathbbm{T}_2$ and $\mathbbm{T}_3$, so summing the matched path distances gives the optimal objective value of Problem~\eqref{eq:distance}, that is, the tree distance; \textbf{(b)} follows from Equation \eqref{eqn:matched-paths}; \textbf{(c)} follows from the fact that $p^1_i \leftrightarrow q^3_i$ is the optimal matching between $\mathbbm{T}_1$ and $\mathbbm{T}_3$ whereas $p^1_i \leftrightarrow p^3_i$ is not; \textbf{(d)} follows from the definition of the tree distance (Problem~\eqref{eq:distance}). Therefore, $d(\mathbbm{T}_1, \mathbbm{T}_2)$ satisfies the triangle inequality (axiom \textbf{(iv)}) and is a metric. \hfill $\square$
\end{proof}

\subsection{Proof of Corollary~\ref{cor:relaxation}} \label{appx:proofs-relaxation}

\begin{proof}{Proof of Corollary~\ref{cor:relaxation}.}
    Given trees $\mathbbm{T}_1$ and $\mathbbm{T}_2$ with $T_1 > T_2$, we append $T_1-T_2$ ``dummy paths'' to $\mathbbm{T}_2$,  which we denote by $\mathcal{D}(\mathbb{T}_2)$, such that the distance from any dummy path to any path $p \in \mathcal{P}(\mathbb{T}_1)$ is equal to $w(p)$. That is, $d(p,q)=w(p),\ \forall p \in \mathcal{P}(\mathbb{T}_1), q \in \mathcal{D}(\mathbb{T}_2)$. We refer to this representation for Problem~\eqref{eq:distance} as Problem~\eqnum\label{eq:distance-reform}. We then rewrite the linear relaxation of Problem~\eqref{eq:distance-reform} as:
    \begin{equation} \label{eq:distance-reform-relax}
        \begin{split}
            d(\mathbbm{T}_1, \mathbbm{T}_2)
            \ =\ \min_{\bm x} & \quad \sum_{p \in \mathcal{P}(\mathbb{T}_1)} \sum_{q \in \mathcal{P}(\mathbb{T}_2) \cup \mathcal{D}(\mathbb{T}_2)} d(p,q) x_{pq} \\
            \text{s.t.} & \quad \sum_{q \in \mathcal{P}(\mathbb{T}_2)} x_{pq} = 1, \quad \forall p \in \mathcal{P}(\mathbb{T}_1) \\
            & \quad \sum_{p \in \mathcal{P}(\mathbb{T}_1)} x_{pq} = 1, \quad \forall q \in \mathcal{P}(\mathbb{T}_2) \\
            & \quad x_{pq} \geq 0, \qquad \forall p \in \mathcal{P}(\mathbb{T}_1), \quad \forall q \in \mathcal{P}(\mathbb{T}_2) \cup \mathcal{D}(\mathbb{T}_2)
        \end{split}
    \end{equation}
    The coefficient matrix of Problem~\eqref{eq:distance-reform-relax} is totally unimodular and thus every extreme point of the feasible region is integral. Moreover, since Problem~\eqref{eq:distance-reform-relax} is a linear optimization problem, the optimum will be attained at an extreme point. Therefore, there exists an optimum to Problem~\eqref{eq:distance-reform-relax} that is integral. Since Problem~\eqref{eq:distance-reform-relax} is more constrained than Problem~\eqref{eq:distance-reform}, the integral optimum to the former also solves the latter, as well as the original Problem~\eqref{eq:distance}. The cases with $T_1 \leq T_2$ can be proved similarly. \hfill $\square$
\end{proof}

\subsection{Proof of Proposition~\ref{lem:upperbound}} \label{appx:proofs-upperbound}

\begin{proof}{Proof of Proposition~\ref{lem:upperbound}.}
    The maximum number of paths in a tree of depth $D$ is $2^D$. The maximum distance between two paths of depth $D$ is $2D+\lambda$: it occurs when the $D$ splits are performed on $2D$ distinct features, each split is performed on values that are $\epsilon$-close to the upper or lower bound for the corresponding feature (with $\epsilon \rightarrow 0$), and the resulting class labels are also different. Provided that the number of features and classes are large enough, we can construct two trees that achieve this upper bound. \hfill $\square$
\end{proof}

\section{Extended Sensitivity Analysis of the Proposed Tree Distance} \label{appx:numerical-experiments-distance}

We present the full numerical study of the key properties of the proposed tree distance metric, which we briefly described in Section \ref{ss:numerical-experiments-distance}. We use the case studies of Section \ref{s:case-studies}. We test the sensitivity of the proposed distance metric to two types of perturbations --- direct and indirect ones. Direct perturbations refer to immediate interventions and changes in the tree structure. Indirect perturbations refer to modifications in the training data. In the remainder of this section, we more concretely describe the methodology we use to obtain each type of perturbation and discuss the corresponding results, shown in Figure \ref{fig:distance}.

To investigate the sensitivity of the proposed distance metric to direct perturbations, we repeat the following process 100 times independently for each case study. In each repetition, we train a decision tree model using 5-fold cross-validation. We then randomly perturb each threshold $t$ in the tree by a percentage $\theta$ drawn uniformly at random from $[0,\theta_{\max}]$, i.e., $t_{\text{pert}}=t\cdot(1+2\cdot \theta_{\max} \cdot \theta - \theta_{\max})$. We vary $\theta_{\max} \in [0,1]$ and measure the distance between the original and the perturbed trees in each case. We report the mean and standard deviation of the distance, obtained across all repetitions and case studies. Figure \ref{fig:distance}(left) suggests that, as the amount of perturbation increases, the distance increases monotonically. The relationship is linear and the mean distance, expressed as a percentage of the maximum possible distance between any two trees of the given depth, varies between 1.5\%, when the maximum perturbation is 10\%, and 12.5\%, when the maximum perturbation is 100\%. 

We proceed to study the sensitivity of the proposed distance metric to indirect perturbations, which refer to modifications in the training data. We repeat the following process 10 times for each case study. In each repetition, we randomly permute the data and keep the first half of them to train a decision tree model using 5-fold cross-validation. Then, for $\theta \in [0.2,\dots,1.]$, we sequentially replace a $\theta$ fraction of the first half of the data with the same amount of data taken from the second half, we train a new decision tree model, and we measure the distance between the two. We report the mean and standard deviation of the distance, obtained across all repetitions and case studies. Figure \ref{fig:distance}(right) shows that, as $\theta$ increases, the distance has an increasing trend, albeit non-monotonic. We attribute this lack of monotonicity to the inherent instability of decision tree models \citep{breiman1996heuristics}, rather than to the proposed distance metric. We remark that the distance percentages in this experiment are low because the upper bound on the distance was obtained using the largest maximum allowable tree depth examined during cross-validation (rather than the actual depth of any of the two trees).

\begin{figure}[!ht]
\begin{subfigure}{0.48\textwidth}
\includegraphics[width=\linewidth]{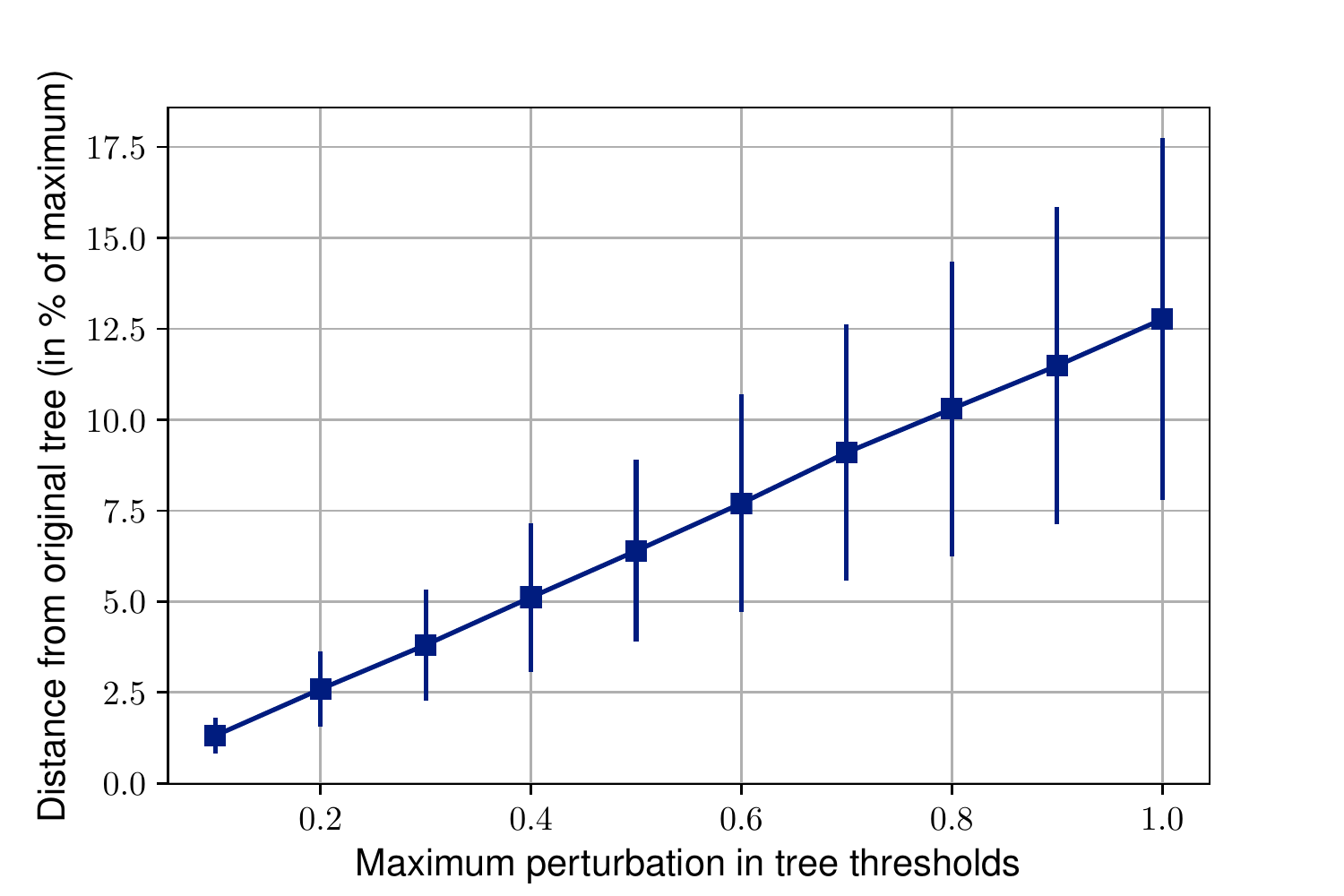} 
\end{subfigure}\hspace*{\fill}
\begin{subfigure}{0.48\textwidth}
\includegraphics[width=\linewidth]{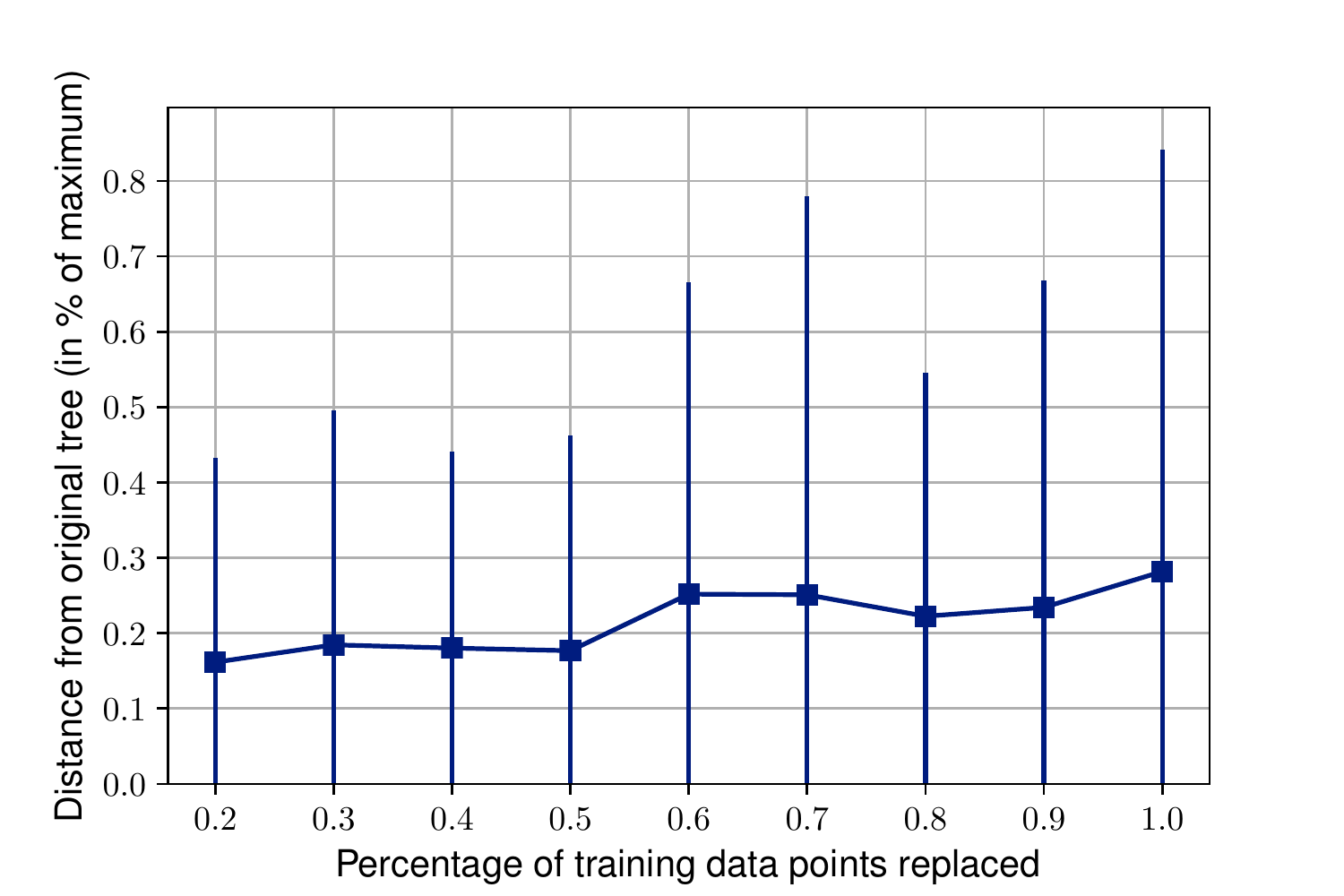} 
\end{subfigure}
\caption{Mean and standard deviation of the distance between the original tree and perturbed tree. The perturbed trees are obtained by randomly perturbing the split thresholds (left) or replacing a fraction of the training points (right). We aggregate across multiple repetitions (i.e., perturbations) for each dataset. } \label{fig:distance}
\end{figure}

\end{APPENDICES}

\end{document}